\documentclass{article}

\usepackage{microtype}
\usepackage{graphicx}
\usepackage{subfigure}
\usepackage{booktabs} 
\usepackage{mathrsfs}

\usepackage{thm-restate}

\usepackage{nicefrac}
\usepackage{xcolor}
\usepackage{amsmath,amssymb,amsthm}

\theoremstyle{definition}

\usepackage{enumitem}
\setlist[enumerate]{leftmargin=0.5cm,topsep=0pt,itemsep=-2pt}
\setlist[itemize]{leftmargin=0.5cm,topsep=0pt,itemsep=-2pt}

\usepackage{mathrsfs}
\newcommand{\probspace}{\mathscr{P}}
\newcommand{\signedspace}{\mathscr{M}}

\usepackage{mathtools}
\usepackage{hyperref}
\hypersetup{
    colorlinks=true,
    linkcolor=blue,
    citecolor=cyan,
}

\usepackage[accepted]{icml2021}

\icmltitlerunning{Off-policy Distributional Q($\lambda$): Distributional RL without Importance Sampling}

\begin{document}

\twocolumn[
\icmltitle{Off-policy Distributional Q($\lambda$): \\ Distributional RL without Importance Sampling}



\icmlsetsymbol{equal}{*}

\begin{icmlauthorlist}
\icmlauthor{Yunhao Tang}{dm}
\icmlauthor{Mark Rowland}{dm}
\icmlauthor{R\'emi Munos}{dm}
\icmlauthor{Bernardo \'Avila Pires}{dm}
\icmlauthor{Will Dabney}{dm}
\end{icmlauthorlist}

\icmlaffiliation{dm}{Google DeepMind}

\icmlcorrespondingauthor{Yunhao Tang}{robintyh@google.com}

\icmlkeywords{Machine Learning, ICML}

\vskip 0.3in
]



\printAffiliationsAndNotice{\icmlEqualContribution} 

\begin{abstract}
    We introduce off-policy distributional Q($\lambda$), a new addition to the family of off-policy distributional evaluation algorithms.  Off-policy distributional Q($\lambda$) does not apply importance sampling for off-policy learning, which introduces intriguing interactions with signed measures. Such unique properties distributional Q($\lambda$) from other existing alternatives such as distributional Retrace. 
     We characterize the algorithmic properties of distributional Q($\lambda$) and validate theoretical insights with tabular experiments. We show how distributional Q($\lambda$)-C51, a combination of Q($\lambda$) with the C51 agent, exhibits promising results on deep RL benchmarks.
\end{abstract}

\section{Introduction}

Random returns $\sum_{t=0}^\infty \gamma^t R_t$ are of fundamental importance to reinforcement learning (RL). While value-based RL focuses on learning the expectation of random returns $\mathbb{E}\left[\sum_{t=0}^\infty \gamma^t R_t\right]$ \citep{sutton1998}, distributional RL has demonstrated benefits to approximate the full distribution of the random return \citep{bellemare2017distributional}.

There is a continuous spectrum of algorithms for learning return distributions of a target policy. At one end of the spectrum is Monte-Carlo simulation, where one generates the full path of returns $(R_t)_{t=0}^\infty$ along a trajectory, building a direct estimate of the random return \citep{bdr2022}. At another end of the spectrum lies one-step temporal difference (TD) algorithms, where one approximate the random return by bootstrapping from the next state distribution \citep{bellemare2017distributional}. The bias and variance trade-off between these two extreme cases are akin to their counterparts in the case of value-based RL \citep{sutton1988learning}.

By balancing the bias-variance trade-off, the best performing algorithm is usually found by interpolating the two extremes. Combining full Monte-Carlo simulation and one-step distributional TD learning, one can obtain distributional Retrace, a  family of multi-step distributional learning algorithms \citep{tang2022nature}. For on-policy case where the data generation policy is the same as the target policy, distributional Retrace recovers a distributional equivalent of TD($\lambda$) \citep{sutton1988learning,nam2021gmac}. For the off-policy case, distributional Retrace can also approximate the target distribution efficiently, by adjusting for the discrepancy between the data collection policy and target policy using importance sampling. Prior work has established the efficiency of such multi-step distributional RL algorithms in large-scale practices, which have enabled significant improvements sin agent performance \citep{gruslys2017reactor,tang2022nature}.

Importance sampling is  fundamental to off-policy distributional RL, and to off-policy RL in general \citep{precup2001off,munos2016safe,espeholt2018impala}. By applying a careful reweighting with the probability ratios, it allows for learning from an off-policy trajectory \emph{as if} it were generated on-policy. Importance sampling has a number of critical limitations: it often introduces high variance; it is not applicable when the 
probability of the data collection policy is not available, which is the case for many applications.

In this work, we introduce off-policy distributional Q($\lambda$), a multi-step distributional RL algorithm \emph{without} the need for importance sampling. Off-policy distributional Q($\lambda$) draws inspirations from value-based off-policy Q($\lambda$) \citep{harutyunyan2016q} and adopts a single trace coefficient $\lambda\in[0,1]$ to mediate various properties of the algorithm, such as the bias and variance trade-off. Without importance sampling, distributional Q($\lambda$) is fundamentally different from distributional Retrace. Below, we highlight a few intriguing and important properties of off-policy distributional Q($\lambda$).

\begin{figure*}[t]
    \centering
    \includegraphics[keepaspectratio,width=.95\textwidth]{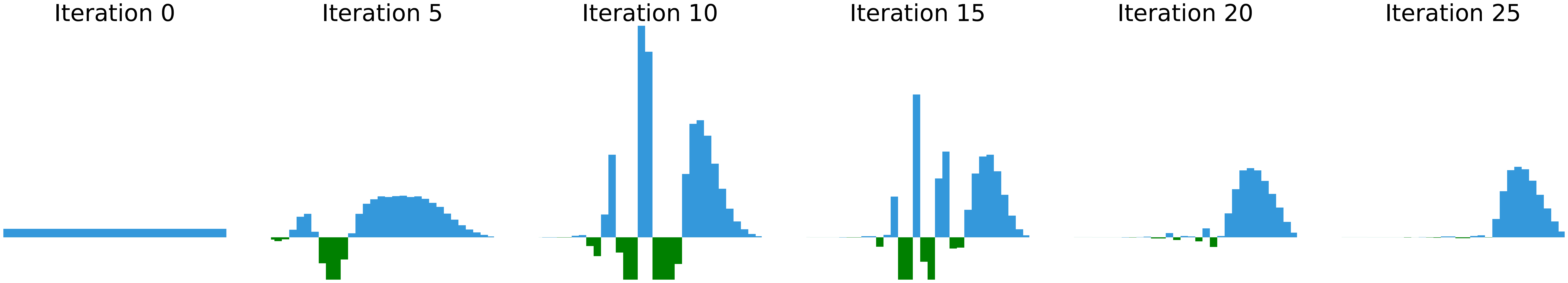}
    \caption{An illustration of the signed measure properties specific to the off-policy Q($\lambda$) operator. The blue and green bars represent the positive and negative probability masses of unit mass signed measures. We visualize the iterate $\eta_{k+1}=\mathcal{A}_\lambda^{\pi,\mu}\eta_k$ for a fixed state-action pair over time on a tabular MDP. The iterate starts as a distribution (an element in $\mathcal{P}(\mathbb{R})$, transitions into a signed measure with unit mass (an element in $\mathcal{M}_1(\mathbb{R})$, and eventually converge to the target return distribution $\eta^\pi$, which is itself a distribution. Any prior distributional RL policy evaluation operators will not exhibit such intriguing behavior, as their iterates are always distributions.}
    \label{fig:iterations}
\end{figure*}

\paragraph{Contraction and fixed point.} By design, off-policy distributional Q($\lambda$) operator has the target return distribution as a fixed point. The operator is also contractive when the target and the data collection policy is close enough, i.e., when the data is not too off-policy; we provide precise characterizations in Section~\ref{sec:qlambda}. As a result when these conditions are met, dynamic programming based or sample based  distributional Q($\lambda$) will converge to the target distribution. When on-policy, distributional Q($\lambda$) reduces to the distributional equivalent of value-based TD($\lambda$) or Q($\lambda$).

\paragraph{Signed measures and representations.} A distinguishing feature of off-policy distributional Q($\lambda$) is that applications of the operator produces signed measures. This contrasts with prior operators such as distributional Retrace \citep{tang2022nature} or one-step Bellman operator \citep{bellemare2017distributional}, where the iterates are naturally confined to be distributions. Intriguingly, this also implies that a convergent algorithm would require representing signed measures. Intuitively, this feels like an unnecessary burden since the fixed point itself is just a distribution. We note that this is the result of unique interaction between distributional learning and off-policy learning without importance sampling. See Figure~\ref{fig:iterations} for an illustration of how the signed measure iterates evolve under the operator: it starts as a distribution, evolves into signed measures during intermediary iterations, and finally returns back to a distribution.

To derive a practical algorithm based on the operator, we introduce an extension of the categorical distributional RL algorithm \citep{rowland2018analysis} to approximations the signed measure iterates produced by distributional Q($\lambda$) (Section~\ref{sec:categorical}).

\paragraph{Trust region interpretation and deep RL.} The contraction property of off-policy Q($\lambda$) naturally introduces a form of trust region constraint between target and data collection policy, from which we derive a heuristic to adapt the target policy for optimal control. This new implementation, in combination with C51 \citep{bellemare2017distributional}, improves over both baseline off-policy Q($\lambda$) and distributional Retrace over Atari-57 benchmarks.

\section{Background}
\label{sec:background}

Consider a Markov decision process (MDP) represented as the tuple $\left(\mathcal{X},\mathcal{A},P_R,P,\gamma\right)$ where $\mathcal{X}$ is the state space, $\mathcal{A}$ the action space, $P_R:\mathcal{X}\times\mathcal{A}\rightarrow \probspace(\mathscr{R})$ the reward kernel (with $\mathscr{R}$ a finite set of possible rewards), $P:\mathcal{X}\times\mathcal{A}\rightarrow\probspace(\mathcal{X})$ the transition kernel and $\gamma\in [0,1)$ the discount factor. In general, we use $\probspace(A)$ denote a distribution over set $A$.
We assume the reward to take a finite set of values mainly because it is notationally simpler to present results; it is straightforward to extend our results to the general case. We also assume the rewards are bounded $R_t\in[R_\text{min},R_\text{max}]$.

Throughout, we let $\pi:\mathcal{X}\rightarrow\probspace(\mathcal{A})$ be a fixed policy and use $(X_t,A_t,R_t)_{t=0}^\infty\sim\pi$ to denote a random trajectory sampled from $\pi$, such that $A_t\sim \pi(\cdot|X_t),R_t\sim P_R(\cdot|X_t,A_t),X_{t+1}\sim P(\cdot|X_t,A_t)$. Define $G^\pi(x,a)\coloneqq\sum_{t=0}^\infty \gamma^t R_t$ as the random return, obtained by following $\pi$ starting from $(x,a)$. The Q-function $Q^\pi(x,a)\coloneqq\mathbb{E}[G^\pi(x,a)]$ is defined as the expected return under policy $\pi$. For convenience, we also adopt the vector notation $Q\in\mathbb{R}^{\mathcal{X}\times\mathcal{A}}$.
Define the one-step value-based Bellman operator $T^\pi:\mathbb{R}^{\mathcal{X\times\mathcal{A}}}\rightarrow\mathbb{R}^{\mathcal{X\times\mathcal{A}}}$ such that $T^\pi Q(x,a) \coloneqq \mathbb{E}[R_0+\gamma Q\left(X_1,A_1^\pi\right)|X_0=x,A_0=a]$ where $Q(X_t,A_t^\pi)\coloneqq\sum_a \pi(a|X_t)Q(X_t,a)$.  The Q-function $Q^\pi$ satisfies  $Q^\pi=T^\pi Q^\pi$ and is also the unique fixed point of $T^\pi$.

\begin{table*}
\centering
    \caption{A comparison between different distributional operators for policy evaluation with target policy $\pi$. The distributional on-policy Q($\lambda$) is an extension of the value-based on-policy Q($\lambda$) operator to the distributional case; its details can be found in Appendix~\ref{appendix:alternative}. All operators preserve the space of probability distribution vector $\probspace(\mathbb{R})^{\mathcal{X}\times\mathcal{A}}$ except off-policy distributional Q($\lambda$). Off-policy distributional Q($\lambda$) is contractive when $\pi$ and $\mu$ are close enough (Lemma~\ref{lemma:qlambdacontraction}) and has $\eta^\pi$ as the unique fixed point.
    \newline}
\begin{small}
\begin{sc}
 \begin{tabular}{l|c|c|c}\toprule[1.5pt]
 Distributional operators & Closed for which space  & Contraction rate under $\bar{\ell}_p$ & Fixed point
\\\midrule
Dist. Bellman operator $\mathcal{T}^\pi$ & \bf $\probspace(\mathbb{R})^{\mathcal{X}\times\mathcal{A}}$  & $\gamma^{1/p}$ &  \bf $\eta^\pi$  \\
 Dist. on-policy Q($\lambda$) $\mathcal{T}_\lambda^\pi$ & \bf $\probspace(\mathbb{R})^{\mathcal{X}\times\mathcal{A}}$ & $\left(\frac{\gamma(1-\lambda)}{1-\lambda\gamma}\right)^{1/p}$ &  \bf $\eta^\pi$  \\
 Dist. Retrace  $\mathcal{R}^{\pi,\mu}$ & \bf $\probspace(\mathbb{R})^{\mathcal{X}\times\mathcal{A}}$ & $[0,\gamma^{1/p}]$ &  \bf $\eta^\pi$  \\
 Dist. off-policy Q($\lambda$)  $\mathcal{A}_\lambda^{\pi,\mu}$ & \bf $\signedspace_1(\mathbb{R})^{\mathcal{X}\times\mathcal{A}}$   & $\beta_p$ in Lemma~\ref{lemma:qlambdacontraction} &  $\eta^\pi$ if contractive \\
\bottomrule
\end{tabular}
\end{sc}
\end{small}
\vskip -0.1in
\label{table:dist-ops}
\end{table*}

\subsection{Multi-step value-based learning and off-policy Q($\lambda$)}

In value-based learning, the fixed point $Q^\pi$ can be obtained by repeatedly applying the Bellman operator $T^\pi$: $Q_{k+1}=T^\pi Q_k$ such that the sequence of iterate $Q_k$ converges to $Q^\pi$ at a rate of $\gamma^k$. To accelerate the convergence, Q($\lambda$) proposes the geometrically weighted average \citep{sutton1998} across multi-step bootstrapped targets. Define  $\delta_t^\pi\coloneqq R_t+\gamma Q\left(X_{t+1},A_{t+1}^\pi\right) - Q(X_t,A_t)$ as the value-based TD error, the Q($\lambda$) back-up target is
\begin{align*}
    T_\lambda^\pi Q(x,a) = Q(x,a) + \mathbb{E}_{\pi}\left[ \sum_{t=0}^\infty \lambda^t\gamma^t\delta_t^\pi\right].
\end{align*}
The Q($\lambda$) operator $T_\lambda^\pi$ is $\frac{\gamma(1-\lambda)}{1-\lambda\gamma}$-contractive and improves upon the one-step Bellman operator. The Q($\lambda$) operator $T_\lambda^\pi$ is on-policy, as the above expectation is taken under the target policy $\pi$. In off-policy learning,  the data is generated under a behavior policy $\mu$ (i.e., the data collection policy), which generally differs from the target policy $\pi$. As a standard assumption, we require $\text{supp}(\pi(\cdot|x))\subseteq\text{supp}(\mu(\cdot|x)),\forall x\in\mathcal{X}$. Off-policy Q($\lambda$) is a straightforward extension of Q($\lambda$) to the off-policy case \citep{harutyunyan2016q}, whose back-up target is now 
\begin{align*}
    A_\lambda^{\pi,\mu} Q(x,a) = Q(x,a) + \mathbb{E}_{\mu}\left[ \sum_{t=0}^\infty \lambda^t\gamma^t\delta_t^\pi\right].
\end{align*}
By construction, the operator $A_\lambda^{\pi,\mu}$ has $Q^\pi$ as a fixed point. Unlike importance sampling (IS) based methods such as Retrace \citep{precup2001off, munos2016safe}, off-policy Q($\lambda$) does not apply IS for off-policy corrections. As a result, $A^{\pi, \mu}_\lambda$ is generally only a contraction under certain conditions on $\pi$, $\mu$, and $\lambda$ \citep{harutyunyan2016q}

\subsection{Distributional reinforcement learning}

In general, the return $G^\pi(x,a)$ is a random variable and we define its distribution as $\eta^\pi(x,a)\coloneqq \text{Law}_\pi\left(G^\pi(x,a)\right)$. 
The return distribution satisfies the distributional Bellman equation \citep{morimura2010nonparametric,morimura2012parametric,bellemare2017distributional,rowland2018analysis,bdr2022},
\begin{align}
     \eta^\pi(x,a) = \mathbb{E}_\pi\left[ \left(\textrm{b}_{R_0,\gamma}\right)_{\#} \eta^\pi\left(X_1,A_1^\pi\right)\; \middle| \; X_0=x,A_0=a\right] \, , \label{eq:dist-bellman}
\end{align}
where $(\textrm{b}_{r,\gamma})_\#:\probspace(\mathbb{R})\rightarrow \probspace(\mathbb{R})$ is the pushforward operation defined through the function $\textrm{b}_{r,\gamma}(z)=r+\gamma z$ \citep{rowland2018analysis}. For convenience, we adopt the notation $\eta^\pi(X_t,A_t^\pi)\coloneqq\sum_a \pi(a|X_t)\eta^\pi(X_t,a)$.
Throughout the paper, we focus on the space of distributions with \emph{bounded} support.

Let $\eta\in\probspace(\mathbb{R})^{\mathcal{X}\times\mathcal{A}}$ be any return distribution function, we define the \emph{distributional Bellman operator} $\mathcal{T}^\pi: \probspace_\infty(\mathbb{R})^{\mathcal{X}\times\mathcal{A}}\rightarrow\probspace_\infty(\mathbb{R})^{\mathcal{X}\times\mathcal{A}}$ as follows \citep{rowland2018analysis,bdr2022},
\begin{align}
    \mathcal{T}^\pi\eta(x,a)\coloneqq \mathbb{E}\left[(\textrm{b}_{R_0,\gamma})_\# \eta(X_1,A_1^\pi)\; \middle| \; X_0=x,A_0=a\right]\, .\label{eq:dist-bellman-op}
\end{align}
Let $\eta^\pi$ be the collection of return distributions under $\pi$; 
the distributional Bellman equation can then be rewritten as
$
    \eta^\pi = \mathcal{T}^\pi \eta^\pi
$. The distributional Bellman operator $\mathcal{T}^\pi$ is $\gamma^{1/p}$-contractive under the $\ell_p$ distance \citep{rowland2018analysis,bdr2022} for any $p\geq 1$, so that $\eta^\pi$ is its unique fixed point.

As a remark for technically minded readers, we note that since $\ell_p$ is initially defined between CDFs of the distributions, it is naturally extended between signed measures too. Meanwhile, it is more challenging to extend Wasserstein distance, another commonly used metric in distributional RL \citep{bdr2022}, to signed measures. Hence all the results in this work are stated in terms of the $\ell_p$ distance.

\subsection{Multi-step distributional RL}

The multi-step distributional bootstrapping bears qualitative differences from value-based multi-step learning \citep{gruslys2017reactor,tang2022nature}. In off-policy learning, let $\rho_t\coloneqq \pi(A_t|X_t)/\mu(A_t|X_t)$ be the step-wise importance sampling (IS) ratio at time step $t$. Let $c_t\in [0,\rho_t]$ be a time-dependent trace coefficient. We denote $c_{1:t}=c_1\cdots c_t$ and define $c_{1:0}=1$ by convention. \citet{tang2022nature} shows that 
distributional Retrace operator $\mathcal{R}^{\pi,\mu}:\probspace(\mathbb{R})^{\mathcal{X}\times\mathcal{A}}\rightarrow \probspace(\mathbb{R})^{\mathcal{X}\times\mathcal{A}}$ is
\begin{align}
    \mathcal{R}^{\pi,\mu}\eta(x,a) \coloneqq \eta(x,a) + \mathbb{E}_{\mu}\left[ \sum_{t=0}^\infty c_{1:t} \cdot \left(\textrm{b}_{G_{0:t-1},\gamma^{t}}\right)_{\#} \Delta_t^\pi  \right],\label{eq:dist-retrace-operator}
\end{align}
where $\Delta_t^\pi=\mathcal{T}^\pi\eta(X_t,A_t)-\eta(X_t,A_t)$ is the distributional one-step TD error. Distributional Retrace has $\eta^\pi$ as the unique fixed point and in general contracts faster than the one-step distributional Bellman operator $\mathcal{T}^\pi$. 

\section{Off-policy distributional Q($\lambda$)}
\label{sec:qlambda}

We now discuss a number of essential properties of off-policy distributional Q($\lambda$) operator: its origin of derivation, its fixed point and contraction property, and its unique interaction with signed measures.

There are a few equivalent ways to arrive at the operator: to better draw connections to existing multi-step off-policy operators, we start with the mathematical form of the distributional Retrace operator in Eqn~\eqref{eq:dist-retrace-operator}, and define distributional Q($\lambda$) operator with the trace coefficient $c_t=\lambda\in[0,1]$:
\begin{align}
    \mathcal{A}_\lambda^{\pi,\mu}\eta(x,a) \coloneqq \eta(x,a) + \mathbb{E}_{\mu}\left[ \sum_{t=0}^\infty \lambda^t \cdot \left(\textrm{b}_{G_{0:t-1},\gamma^{t}}\right)_{\#} \Delta_t^\pi  \right].\label{eq:dist-qlambda-operator}
\end{align}
The off-policy distributional Q($\lambda$) is \emph{not} a special case of distributional Retrace, despite their apparent similarities. Critically, distributional Retrace requires the trace coefficient $c_t\in[0,\rho_t]$ to ensure conservative trace cutting \citep{munos2016safe}, whereas setting $c_t=\lambda$ can violate such a condition. A detailed derivation of off-policy distributional Q($\lambda$) can extend from the distributional on-policy Q($\lambda$) operator \citep{nam2021gmac}, similar to how one arrives at value-based off-policy Q($\lambda$) from on-policy Q($\lambda$) in Section~\ref{sec:background}. We detail such a derivation in Appendix~\ref{appendix:alternative}.

Before  we will elaborate on the fundamental differences between off-policy distributional Q($\lambda$) and previous distributional operators: interaction with signed measures.

\subsection{Off-policy dist. Q($\lambda$) targets are signed measures}

To facilitate the discussion, we introduce the notation $\signedspace_1(\mathbb{R})$ of the space of signed measures with total mass of $1$.
This particular space of signed measure is a natural generalization of (and a superset to) the space of probability measures by allowing for negative probability mass in certain locations of the distribution, while still requiring a unit total mass. See \citet{bdr2022} for a more formal definition of the signed measure space.

While previous distributional operators such as the one-step operator $\mathcal{T}^\pi$ and the Retrace operator $\mathcal{R}^{\pi,\mu}$ map within the space of distributions, this is not the case for the off-policy distributional Q($\lambda$) operator. Concretely, it is possible to find a vector of distribution $\eta\in\probspace(\mathbb{R})^{\mathcal{X}\times\mathcal{A}}$ such that
$\mathcal{A}_\lambda^{\pi,\mu}\eta(x,a)$ is not a distribution (but a signed measure in general). In other words, the space of probability distribution is not closed under the operator $ \mathcal{A}_\lambda^{\pi,\mu}$. Fortunately, the space of signed measure is closed.
\begin{restatable}{lemma}{lemmasigned}\label{lemma:signed} (\textbf{Closeness of the space of signed measures}) Given any $\eta\in\signedspace_1(\mathbb{R})^{\mathcal{X}\times\mathcal{A}}$, 
we have $\mathcal{A}_\lambda^{\pi,\mu}\eta\in\signedspace_1(\mathbb{R})^{\mathcal{X}\times\mathcal{A}}$.
\end{restatable}

Figure~\ref{fig:iterations} also illustrate the closeness property of the operator: all iterates are unit mass signed measures. 

\subsection{Fixed point and contraction property}

We now discuss the fixed point and contraction property of the distributional Q($\lambda$) operator. The technical approach is mostly motivated by the value-based case \citep{harutyunyan2016q}. First note that by construction, the off-policy distributional Q($\lambda$) operator has $\eta^\pi$ as one fixed point.
\begin{restatable}{lemma}{lemmaqlambdafixedpoint}\label{lemma:qlambdafixedpoint} (\textbf{Fixed point}) $\eta^\pi$ is a fixed point of $\mathcal{A}_\lambda^{\pi,\mu}$.
\end{restatable}

Hence, we can evaluate the target return distribution $\eta^\pi$ by an algorithm that converges to the fixed point of the operator. A sufficient condition for the approximate dynamic programming algorithms to converge is that the operator be contractive. We consider contraction under the the $\ell_p$ supremum distance, defined as
\begin{align*}
    \bar{\ell}_p(\eta_1,\eta_2) = \max_{x,a}\ell_p\left(\eta_1(x,a),\eta_2(x,a)\right),
\end{align*}
for any signed measure vectors $\eta_1,\eta_2\in\signedspace_1(\mathbb{R})^{\mathcal{X}\times\mathcal{A}}$. The contraction rate critically depends on the distance between $\pi$ and $\mu$, which we define as $\left\lVert \pi - \mu \right\rVert_1=\max_x \sum_a \left|\pi(a|x) - \mu(a|x)\right|$.

\begin{restatable}{lemma}{lemmaqlambdacontraction}\label{lemma:qlambdacontraction} (\textbf{Contraction}) Let $\epsilon\coloneqq \left\lVert \pi - \mu \right\rVert_1$, then for any $p\geq 1$ and signed measures $\forall\eta_1,\eta_2\in\signedspace_1(\mathbb{R})^{\mathcal{X}\times\mathcal{A}}$,
\begin{align*}
    \bar{\ell}_p\left(\mathcal{A}_\lambda^{\pi,\mu}\eta_1, \mathcal{A}_\lambda^{\pi,\mu}\eta_2\right) \leq \beta_p \bar{\ell}_p\left(\eta_1, \eta_2\right),
\end{align*}
where $\beta_p=\gamma^{1/p}\frac{1-\lambda+\lambda\epsilon}{(1-\lambda)^{(p-1)/p}(1-\lambda\gamma)^{1/p}}$ is the contraction rate under the supremum $\ell_p$ distance. 
\end{restatable}
The contraction rate $\beta_p$ depends on $\left\lVert \pi - \mu \right\rVert_1$ which is a measure of off-policyness; and the value of $\gamma$ and $\lambda$. When $\pi,\mu$ are close enough, the off-policy distributional Q($\lambda$) operator is  contractive.

\begin{restatable}{corollary}{corollarycontraction}\label{corollary:contraction}  When $\left\lVert \pi - \mu \right\rVert_1<\frac{1-\gamma}{\lambda\gamma}$,  we have $\beta_1<1$ and the operator $\mathcal{A}_\lambda^{\pi,\mu}$ is contractive under the $L_1$ distance. This also implies that $\eta^\pi$ is the unique fixed point to $\mathcal{A}_\lambda^{\pi,\mu}$. 
\end{restatable}

A few remarks are in order. We call $\frac{1-\gamma}{\lambda\gamma}$ the contraction radius, the bound on the distance between $\pi$ and $\mu$ to ensure that the operator is contractive,
which also coincides with similar quantities in the value-based off-policy Q($\lambda$) \citep{harutyunyan2016q}. Inverting the radius condition, we derive a bound on the trace coefficient $\lambda$ for the operator to be contractive:  $\lambda<\frac{1-\gamma}{\gamma\left\lVert \pi -\mu\right\rVert_1}$. In other words, when $\pi$ and $\mu$ are far from each other, $\lambda$ can only take small value close to $0$ in order to ensure that the operator is contractive. In such cases,  operator cuts traces very quickly, and can only make use of bootstrapped values in the very near future. Meanwhile, in the limiting on-policy case when $\pi=\mu$, the radius condition is always satisfied and any value of $\lambda\in[0,1]$ makes the operator contractive. The trade-off is that larger value of $\lambda$ will lead to faster contrction in expectation, but induces higher variance.

Note that in general, the radius $\frac{1-\gamma}{\lambda\gamma}$ is fairly conservative because the above conclusion is valid for arbitrary MDPs and arbitrary target and behavior policy. In practice, we find that using larger values of $\lambda$ will often lead to stable learning, i.e., as a result of a contractive operator.

Another intriguing observation is that even though as Lemma~\ref{lemma:signed} showed, the operator $\mathcal{A}_\lambda^{\pi,\mu}$ does not preserve the space of probability distributions, it still has $\eta^\pi$ as the unique fixed point when $\pi$ and $\mu$ are close enough. Consider initializing a distribution vector $\eta_0\in\probspace(\mathbb{R})^{\mathcal{X}\times\mathcal{A}}$ and generate a sequence of iterate based on  $\eta_{k+1}=\mathcal{A}_\lambda^{\pi,\mu}\eta_k$. Lemma~\ref{lemma:signed} suggests that in general, the intermediate iterates $(\eta_k)_{k\geq 1}$ are signed measures. However, as $k\rightarrow\infty$, we can expect the signed measure sequence to converge back to a probability distribution $\eta^\pi$ (see Figure~\ref{fig:iterations} for the illustration). 
This example bears important implications on the parametric representations of intermediate iterates in algorithm designs. By default, one should expect it suffices to represent the iterate in the space of probability distributions, since the fixed point $\eta^\pi$ is itself a probability distribution vector. The case of Q($\lambda$) suggests otherwise: it is necessary to expand the space of representations to signed measures, such that the intermediate iterates $\eta_k$ can be represented.

\paragraph{Alternative way to construct distributional Q($\lambda$).} For interested readers, we also note that there are alternative ways to construct the distributional Q($\lambda$) operator. We provide such a natural alternative in Appendix~\ref{appendix:alternative}, which is closely related to the \emph{path-dependent}  distributional TD errors discussed in \citet{tang2022nature}. Our construction of $\mathcal{A}_\lambda^{\pi,\mu}$ leads to better theoretical properties compared to the alternative.

\section{Learning with categorical representation}\label{sec:categorical}

Return distributions are in general infinite dimensional objects. In practice, it is necessary to approximate the target return distribution with parametric approximations. One commonly used family of parameterization is the categorical representation \citep{bellemare2017distributional,bdr2022}. We provide a brief background below.

\paragraph{Brief background.} In categorical representation, we consider parametric distributions, for a fixed $m\geq 1$, of the form: $\sum_{i=1}^m p_i\delta_{z_i}$, where $(z_i)_{i=1}^m\in\mathbb{R}$ are a fixed set of atoms and $(p_i)_{i=1}^m$ is a categorical distribution such that $\sum_{i=1}^m p_i=1$ and $p_i\geq 0$. For simplicity, we assume $z_i$ to be strictly monotonic $z_i<z_{i+1}$ and the range of atoms covers all possible returns from the MDP $[(1-\gamma)^{-1}R_\text{min},(1-\gamma)^{-1}R_\text{max}]\subset [z_1,z_m]$. Let $\probspace_c(\mathbb{R})$ denote the class of distributions
\begin{align*}
    \probspace_c(\mathbb{R}) \coloneqq \bigg\{ \sum_{i=1}^m p_i\delta_{z_i} | \sum_{i=1}^m p_i=1, p_i\geq 0 \bigg\}.
\end{align*}
When combining parametric representation with the off-policy distributional Q($\lambda$), it is important to account for the fact that signed measures can arise by applying the operator $\mathcal{A}_\lambda^{\pi,\mu}$. We can extend the categorical representation by dropping the non-negativity constraints on $p_i$
\begin{align*}
    \signedspace_{1,c}(\mathbb{R}) \coloneqq \bigg\{ \sum_{i=1}^m p_i\delta_{z_i} | \sum_{i=1}^m p_i=1  \bigg\}.
\end{align*}
Given a target signed measure $\eta\in\mathcal{M}_1(\mathbb{R})$, we define the projection $\Pi_c:\signedspace_\infty(\mathbb{R})\rightarrow \signedspace_{1,c}(\mathbb{R})$ onto the space of categorical distributions by minimizing the $\ell_2$ distance $\Pi_c \eta\coloneqq \arg\min_{\nu\in\signedspace_{1,c}(\mathbb{R})}\ell_2\left(\nu,\eta\right)$. Note a major difference here is that the projection operation also produces a signed measure, which can also be computed in a natural and efficient way \citep{rowland2018analysis,bellemare2017cramer}. Visually, we can understand the categorical projection $\Pi_c\eta$ as a discretized approximation to signed measure $\eta$, see Figure~\ref{fig:categorical-projection} for an illustration. The approximation becomes more accurate as the number of atoms increases.
We refer readers to Chapter 9 of \citet{bdr2022} for further details. 

\subsection{Fixed point and contraction property}

To implement the off-policy distributional Q($\lambda$) in practice, we represent the distribution iterate as a vector of categorical signed measures $\eta_k\in\signedspace_{1,c}(\mathbb{R})^{\mathcal{X}\times\mathcal{A}}$. After applying the operator $\mathcal{A}_\lambda^{\pi,\mu}\eta_k$, we project the back-up target to the space of categorical signed measures. This yields the following recursion
\begin{align}
    \eta_{k+1} = \Pi_c \mathcal{A}_\lambda^{\pi,\mu}\eta_k.\label{eq:categorical-recursion}
\end{align}
Understanding the behavior of the above recursion consists in characterizing the composed operator $\Pi_c \mathcal{A}_\lambda^{\pi,\mu}$. We have the following characterization
\begin{restatable}{lemma}{lemmaprojectedqlambdacontraction}\label{lemma:projectedqlambdacontraction} (\textbf{Contraction of composed operator}) The composed operator $\Pi_c\mathcal{A}_\lambda^{\pi,\mu}$ is $\beta_2$-contractive under the $\bar{L}_2$ distance in the space of signed measure vectors, i.e., $\forall \eta_1,\eta_2\in\signedspace_\infty(\mathbb{R})^{\mathcal{X}\times\mathcal{A}}$,
\begin{align*}
    \bar{L}_2\left(\Pi_c\mathcal{A}_\lambda^{\pi,\mu}\eta_1,\Pi_c\mathcal{A}_\lambda^{\pi,\mu}\eta_2\right)\leq \beta_2\bar{L}_2\left(\eta_1,\eta_2\right),
\end{align*}
where $\beta_2$ is defined in Lemma~\ref{lemma:qlambdacontraction}. The operator is guaranteed to be contractive when $\pi,\mu$ is within the contraction radius defined below
\begin{align*}
    \left\lVert \pi-\mu\right\rVert_1 < \lambda^{-1}\left(\sqrt{(1-\lambda)(\gamma^{-1}-\lambda)} + \lambda - 1\right)
\end{align*}
\end{restatable}

Under the categorical representation, there is an inherent limit on how well one can approximate the true return distribution $\eta^\pi$. The best possible approximation is $\Pi_c\eta^\pi$, the direct projection of the target return to the representation space. The irreducible approximation error is $\bar{L}_2\left(\eta^\pi,\Pi_c\eta^\pi\right)$. When using bootstrapping, the approximation error can compound over time. When $\Pi_c\mathcal{A}_\lambda^{\pi,\mu}$ is contractive, the recursion in Eqn~\eqref{eq:categorical-recursion} converges and let $\eta_\mathcal{A}^\pi$ be the fixed point of the composed operator. We can characterize its approximation error to the target return, by extending \citet[Proposition~3]{rowland2018analysis} below.
\begin{restatable}{lemma}{lemmaprojectedqlambdaerror}\label{lemma:projectedqlambdaerror} (\textbf{Approximation error}) When $\beta_2<1$, we have 
\begin{align*}
    \bar{L}_2 \left(\eta^\pi, \eta_\mathcal{A}^\pi\right) \leq \frac{\bar{L}_2\left(\eta^\pi,\Pi_c\eta^\pi\right)}{\sqrt{1-\beta_2^2}}.
\end{align*}
\end{restatable}

 We see that the parameter $\beta_2$ appear both as the contraction rate as well as a factor in the approximation error. An operator with a fast contraction rate will also have smaller approximation error, which is a recurrent observation made in prior work \citep{bdr2022}. We will validate this theoretical insight in the tabular experiment.

\begin{figure}[t]
    \centering
    \includegraphics[keepaspectratio,width=.45\textwidth]{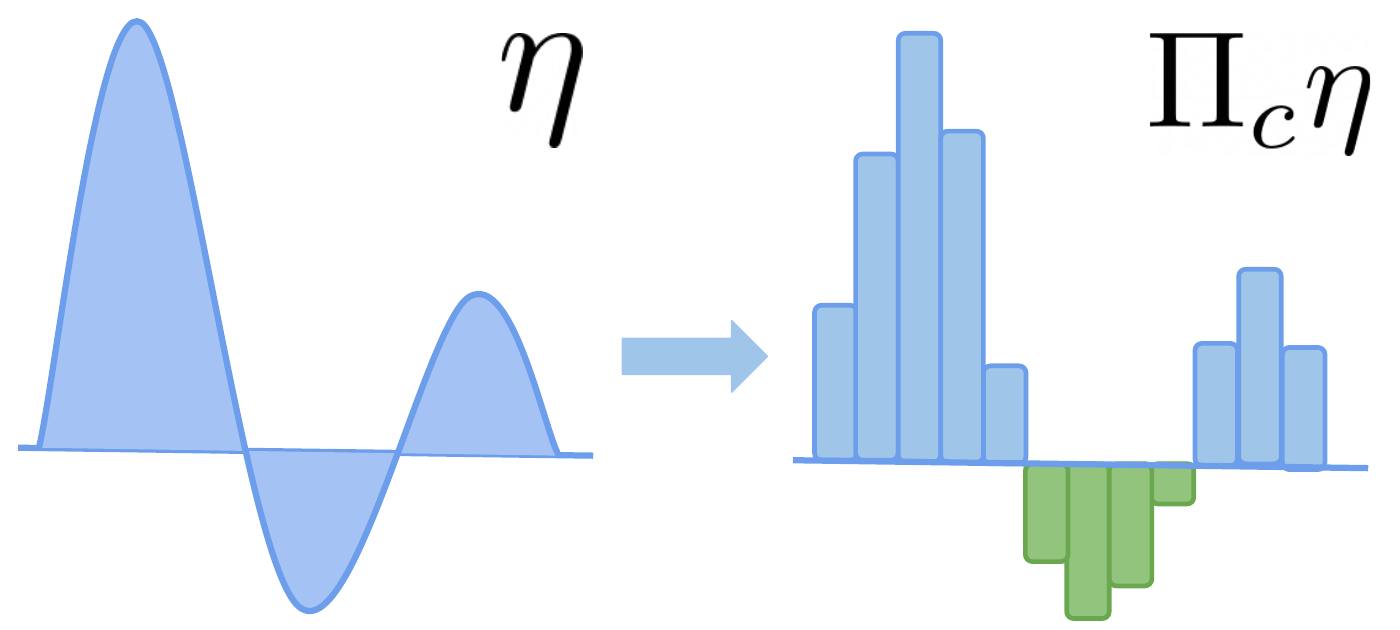}
    \caption{Illustration of categorical projection for the signed measure. On the left, we have a signed measure $\eta\in\signedspace_1(\mathbb{R})$; on the right, we show the categorical projection of the signed measure $\eta$ onto the space $\signedspace_{1,c}(\mathbb{R})$, with green bars showing the negative mass of the projected measure. The categorical projection is a discretized approximation to the original signed measure, with increasing accuracy as the number of atoms $(z_i)_{i=1}^m$ increases.}
    \label{fig:categorical-projection}
\end{figure}

\section{Deep RL implementation: Q($\lambda$)-C51} 
\label{sec:deeprl}

We now describe the deep RL implementation of Q($\lambda$)-C51, an extension of the C51 agent \citep{bellemare2017distributional} to the Q($\lambda$) operator. By design, the agent parameterizes the categorical return distribution $\left(p_i(x,a;\theta)\right)_{i=1}^m$ with neural network parameter $\theta$. Let $\eta_\theta(x,a)=\sum_{i=1}^m p_i(x,a;\theta)\delta_{z_i}$ denote the parameterized categorical distribution. In the original C51 agent, given a target policy $\pi$ and behavior policy $\mu$, the aim is to minimize the KL-divergence between the parameterized categorical distribution and the one-step back-up target

\begin{align*}
    \mathbb{KL}\left(\Pi_c\mathcal{T}^\pi \eta_\theta(x,a), \eta_\theta(x,a)\right),
\end{align*}
where $\mathbb{KL}\left(p,q\right)\coloneqq\sum_i p_i\log q_i / p_i$. To this end, the algorithm carries out gradient descent on the KL-divergence
\begin{align*}
    \theta\leftarrow\theta-\kappa\cdot\nabla_\theta\mathbb{KL}\left(\Pi_c\mathcal{T}^\pi \eta_{\theta^-}(x,a), \eta_\theta(x,a)\right)
\end{align*}
with learning rate parameter $\kappa>0$ and $\theta^-$ is the target network parameter, usually computed as a slow moving average of $\theta$ \citep{mnih2013}. To adapt C51 for off-policy distributional Q($\lambda$), we carry out the gradient update, \emph{heuristically} written as
\begin{align*}
    \theta\leftarrow\theta-\kappa\cdot\nabla_\theta\mathbb{KL}\left(\Pi_c\mathcal{A}_\lambda^{\pi,\mu} \eta_{\theta^-}(x,a), \eta_\theta(x,a)\right).
\end{align*}
Here, \emph{heuristically} refers to the fact that the KL-divergence might not be well defined since $\Pi_c\mathcal{A}_\lambda^{\pi,\mu} \eta_{\theta^-}(x,a)$ can be a sigend measure. To make the implementation more concrete, note that we can always write $\mathcal{A}_\lambda^{\pi,\mu}\eta_{\theta^-}(x,a)$ as a linear combination of proper distributions generated from the future time step
\begin{align*}
    \mathcal{A}_\lambda^{\pi,\mu}\eta_{\theta^-}(x,a) = \mathbb{E}_\mu\left[\sum_{t=0}^\infty w_t \left(\textrm{b}_{G_{0:t-1},\gamma^{t}}\right)_{\#} \eta_{\theta^-}(X_t,A_t)\right] 
\end{align*}
where the combination coefficient $w_t$ can be expressed as
\begin{align*}
    w_t \coloneqq \mathbb{E}_\mu\left[c_1...c_{t-1}\left(\pi(b|X_t) - c(X_t,b)\mu(b|X_t)\right)\right].
\end{align*}
Note that $w_t$ can be negative, whereas for the case of distributional Retrace the coefficient yields the same form but with $w_t\geq 0$; this echos the unique interaction that Q($\lambda$) introduces with signed measures. For more detailed on the derivations of $w_t$, see \citet{tang2022nature} and Appendix~\ref{appendix:proof}. 

Next, we can construct unbiased estimate to the above back-up target by sampling trajectories with the behavior policy $(X_t,A_t,R_t)_{t=0}^\infty \sim \mu$ and calculate the signed measure back-up using the linear combination
\begin{align*}
    \widehat{\mathcal{A}}_\lambda^{\pi,\mu}\eta_{\theta^-}(x,a) = \sum_{t=0}^\infty \widehat{w}_t \underbrace{\left(\textrm{b}_{G_{0:t-1},\gamma^{t}}\right)_{\#} \eta_{\theta^-}(X_t,A_t)}_{\text{proper distribution}},
\end{align*}
where the scalar weights can be computed along the sampled trajectory $\widehat{w}_t=c_1...c_{t-1}\left(\pi(b|X_t) - c(X_t,b)\mu(b|X_t)\right)$. It is straightforward to see that these are unbiased estimates to the coefficients $w_t$. As a result, 
$\widehat{\mathcal{A}}_\lambda^{\pi,\mu}\eta_{\theta^-}(x,a)$ is an unbiased estimate to the signed measure target $\mathcal{A}_\lambda^{\pi,\mu}\eta_{\theta^-}(x,a)$, and is in general a signed measure though each of the unweighted summand $\left(\textrm{b}_{G_{0:t-1},\gamma^{t}}\right)_{\#} \eta_{\theta^-}(X_t,A_t)$ is a proper distribution.
Finally, we compute the gradient update $\widehat{g}_\theta$ as a weighted average of gradients through individual well-defined KL divergences against the prediction $\eta_\theta(x,a)$,
\begin{align}
    \sum_{t=0}^\infty \widehat{w}_t \nabla_\theta\mathbb{KL}\left(\eta_\theta(x,a), \left(\textrm{b}_{G_{0:t-1},\gamma^{t}}\right)_{\#} \eta_{\theta^-}(X_t,A_t)\right) \label{eq:estimate-grad}
\end{align}
See Algorithm 1 for the full algorithmic process.

\subsection{Adapting target policy for optimal control} \label{sec:trust-region}

A practical objective is to maximize the agent performance over time, i.e., optimal control. In this case, we let the target policy be the greedy policy with respect to $Q_{\eta_k}$ where $Q_{\eta_k}$ is the Q-function induced by $\eta_k$. Due to space limit, we state results for optimal control in Appendix~\ref{appendix:optimal}.

One way to interpret the results on policy evaluation is that we can impose constraints on the target policy $\pi_k$ such that it stays within the contraction radius. Concretely, we can let the target policy be a mixture between the greedy policy and behavior policy $\mu$ with $\alpha\in[0,1]$.
\begin{align}
    \pi_k = \alpha \mathcal{G}\left(Q_{\eta_k}\right) + (1-\alpha) \mu \label{eq:mixing-target}
\end{align}
In practice, when the behavior policy is typically defined through a replay buffer, $\mu$ slowly varies over time. In that case, $\alpha$ would be introduced as an extra hyper-parameter.

The above approach is reminiscent of trust region policy optimization \citep{kakade2002approximately,schulman2015}, where the algorithm imposes a trust region constraint over consecutive policy iterates. The contraction radius of distributional Q($\lambda$) or value-based Q($\lambda$) \citep{harutyunyan2016q} can also be understood as a justification to trust region updates, see \citet{tang2020taylor} also for similar discussions.

By regularizing the target policy $\pi_k$ towards the behavior policy, we have made the off-policy evaluation problem more on-policy, which effectively speeds up the contraction rate of the distributional Q($\lambda$) operator. A similar implementation has been considered in \citet{rowland2019adaptive}, with a a similar mixing strategy to speed up the contraction rate of the value-based Retrace algorithm.

\section{Related work}

\paragraph{Distributional Peng's Q($\lambda$).} A closely related operator variant is distributional Peng's Q($\lambda$), adapted from value-based Peng's Q($\lambda$). Since value-based Peng's Q($\lambda$) is a geometrically weighted mixture of $n$-step uncorrected value-based back-ups \citep{peng1994incremental,kozuno2021revisiting}, we can define distributional Peng's Q($\lambda$) as a similar mixture of $n$-step uncorrected distributional-based back-ups. 
\begin{align*}
    \mathcal{P}_\lambda^{\pi,\mu}\eta(x,a) \coloneqq (1-\lambda) \sum_{n=1}^\infty \lambda^{n-1} \mathcal{P}_n^{\pi,\mu}\eta(x,a) 
\end{align*}
where 
$
   \mathcal{P}_n^{\pi,\mu}\eta(x,a)  =  \mathbb{E}_\mu\left[\left(\textrm{b}_{G_{0:n-1},\gamma^{n}}\right)_{\#} \eta(X_n,A_n^\pi)  \right]
$
is the $n$-step uncorrected distributional back-up. In general, since distributional Peng's Q($\lambda$) does not perform off-policy corrections, it enjoys faster contraction but has a fixed point which generally differs from the target fixed point $\eta^\pi$ in off-policy learning. In this work, our focus is on distributional evaluation operators with the correct target fixed point. See Appendix~\ref{appendix:peng} for more details.

\paragraph{Value-based Q($\lambda$) and distributional Retrace.} Many theoretical results on distributional Q($\lambda$) echo value-based Q($\lambda$) \citep{harutyunyan2016q}, such as the contraction radius between target and behavior policy. A key difference between the value-based and distributional setting is the representation; while value-based Q($\lambda$) requires representing a scalar per state-action pair, distributional Q($\lambda$) requires a parameterized signed measure per state-action pair. This latter property also precludes distributional Q($\lambda$) from being a special case of the distributional Retrace  \citep{tang2022nature}.

For technically minded readers, we also discuss in Appendix~\ref{appendix:alternative} how the alternative constructs of distributional Q($\lambda$) differ from that of distributional Retrace.

\section{Experiments}
\label{sec:exp}

\begin{figure}[t]
    \centering
    \includegraphics[keepaspectratio,width=.4\textwidth]{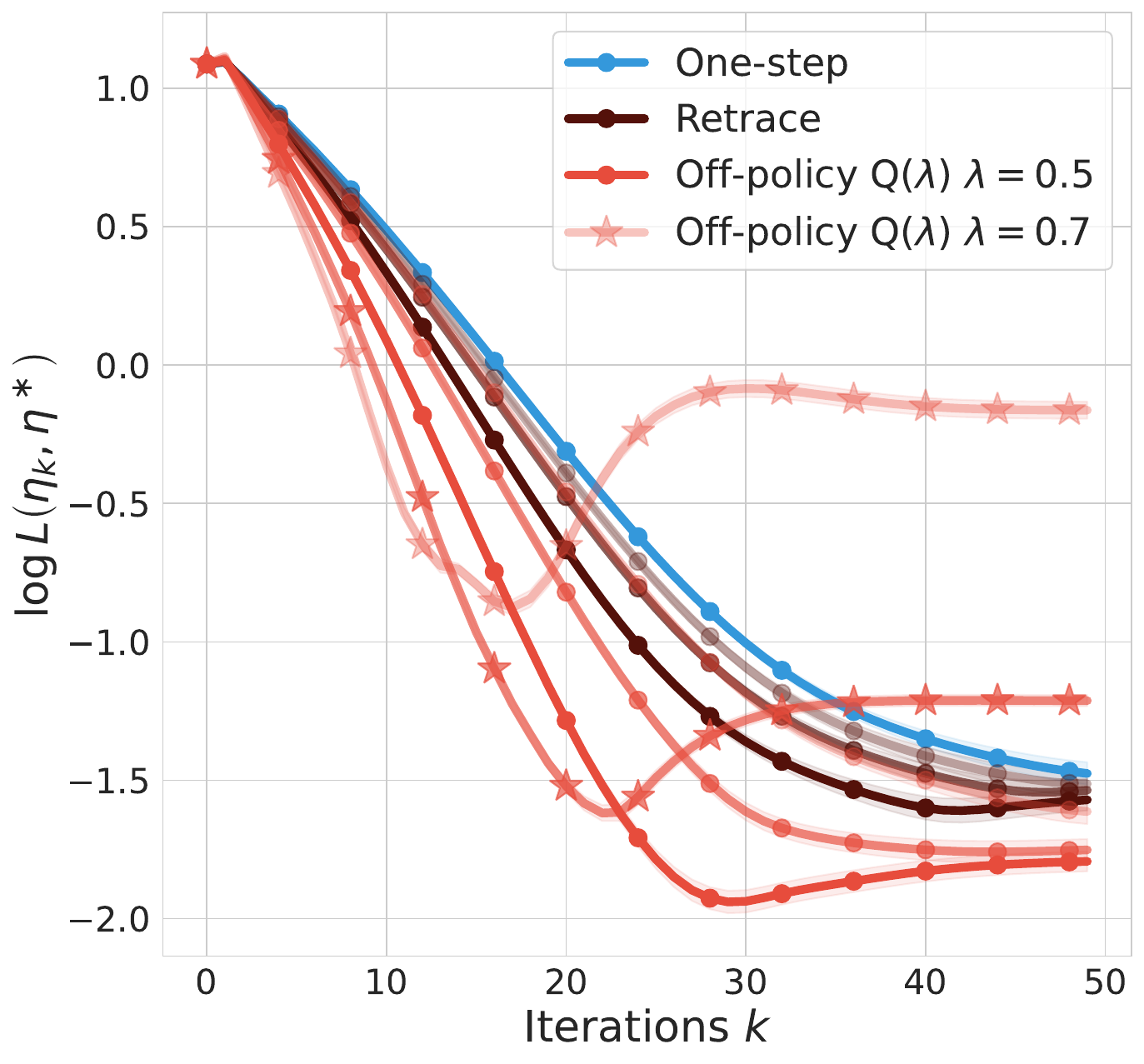}
    \caption{The distance between the algorithmic iterate $\eta_k$ and return distribution for the optimal policy $\eta^\ast$, as we run control algorithms with distributional one-step, Retrace and off-policy Q($\lambda$). All algorithms use categorical representations and set greedy policy as the target policy. Different curves show an algorithmic variant with a different hyper-parameter setting ($\bar{c}$ for Retrace and $\lambda$ for Q($\lambda$)). Note that Q($\lambda$) can obtain better performance than Retrace when $\lambda$ is chosen properly; when $\lambda$ is too large ($\geq 0.7$ in this case), the algorithm diverges -- despite the initial fast decay in the distance, will not converge to the correct fixed point.}
    \label{fig:qlambda_control}
\end{figure}

We start with tabular experiments which validate a number of theoretical insights, followed by an assessment of off-policy Q($\lambda$) in large-scale deep RL experiments.

\begin{figure*}[t]
    \centering
    \includegraphics[keepaspectratio,width=.8\textwidth]{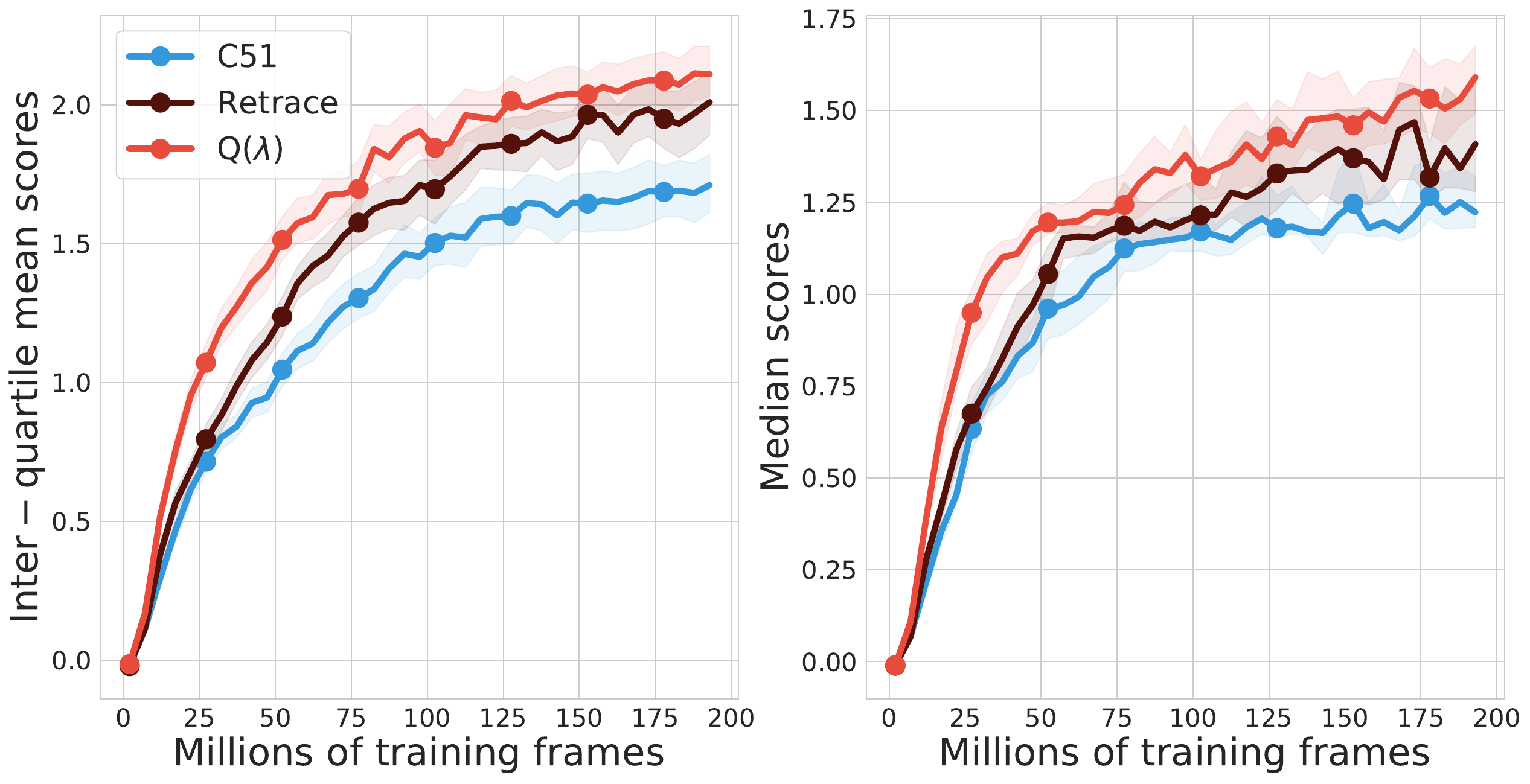}
    \caption{\small{Comparison of C51 \citep{bellemare2017distributional}, Retrace-C51 \citep{tang2022nature} and off-policy distributional Q($\lambda$) with target mixing $\alpha=0.6$ based on Eqn~\eqref{eq:mixing-target} and $\lambda=0.4$. We show the agents' average performance metrics evaluated throughout training: the inter-quartile mean score \citep{agarwal2021deep}, which can be understood as a more robust estimate to the mean score; and the median score, calculated across all 57 games. All scores show the mean and bootstrapped confidence intervals across $5$ seeds \citep{agarwal2021deep}. Off-policy distributional Q($\lambda$) obtains performance improvements over Retrace-C51 when using target mixing.}}
    \label{fig:deeprl}
\end{figure*}

\subsection{Tabular experiments}

We consider a tabular MDP setting with $|\mathcal{X}|=5$ states and $|\mathcal{A}|=20$ actions with discount $\gamma=0.9$. Both the transitions are the reward functions are randomly generated and fixed. The data collection policy $\mu$ is uniform for all time. We compare distributional one-step, Retrace and off-policy Q($\lambda$) with categorical representations for $m=10$, and in the optimal control setting. Throughout, the target policy is the greedy policy induced by the iterate $\eta_k$ with $\eta_{k+1}=\Pi_c\mathcal{R}\eta_k$ where $\mathcal{R}$ is the distributional operator of interest. For Retrace we sweep over the truncation coefficient $\bar{c}\in\{1,2,4\}$ and for Q($\lambda$) over $\lambda\in\{0.1,0.3,0.5,0.7,0.9\}$. Each experiment starts with the same initialization and is repeated $20$ times to show standard errors across seeds. In Figure~\ref{fig:qlambda_control}, we show the distance $L(\eta_k,\eta^\ast)$ as a function of iteration $k$ for different algorithmic variants.

We make a few observations: (1) Both Retrace and off-policy Q($\lambda$) improve over one-step both in terms of asymptotic performance and contraction speed. This corroborates the theoretical insight that the contraction rate and the distance to the final fixed point is related; (2) In this particular case, Q($\lambda$) outperforms Retrace when $\lambda$ is chosen properly. However, when $\lambda$ is too large ($\geq 0.7$ here), the algorithm becomes divergent - despite a fast initial decay in the distance, will converge to a clearly sub-optimal point. Here, a caveat is that since we are testing the case for dynamic programming, we have not considered the variance effect of setting large values for $\bar{c}$ and $\lambda$. Such factors should be accounted for in practice.

Note that with $|\mathcal{A}|=20$, we have made the algorithm effectively more off-policy. When we decrease  $|\mathcal{A}|$, we see that both Retrace and off-policy Q($\lambda$) become better behaved: Q($\lambda$) becomes convergent even with large $\lambda$, and Retrace may outperform Q($\lambda$) by benefiting from the full trace with IS corrections. See more results in Appendix~\ref{appendix:exp}

\subsection{Deep RL experiments}
For the deep RL experiments, we use the Atari game suite of 57 games as the testbed and compare distributional RL agents: C51 \citep{bellemare2017cramer}, Retrace-C51  \citep{tang2022nature} and Q($\lambda$)-C51.
All three algorithms share exactly the same network architecture $\theta$: they parameterize the image inputs via a convnet, followed by MLPs that transform the convnet embeddings into $m$-logits $l(x,a,i),1\leq i\leq m$ per action $a$. The final prediction is computed as a softmax distribution over the logits $p_i(x,a;\theta)\propto \exp(l(x,a,i))$. 
The algorithms and only differ in the back-up targets used for updating the predicted return distribution $\eta_\theta(x,a)$.

For both distributional Retrace and distributional Q($\lambda$), we calculate the back-up targets with partial trajectories of length $n=3$ sampled from the replay buffer. This is consistent with practices in prior work \citep{tang2022nature}. The C51 baseline can be recovered as a special case with $n=1$. In general, the acting policy is usually the $\epsilon$-greedy policy with respect to the Q-function induced by the return distribution $Q_{\eta_\theta(x,a)}$. The value of $\epsilon$ decays over time, in order to achieve a good balance between exploration and exploitation.
The transition tuple $(X_t,A_t,R_t)$ is then put into a replay buffer, which gets sampled when constructing the back-up target. As a result, the effective behavior policy $\mu$ is a mixture of $\epsilon$-greedy policy over time. By default, the target policy is greedy with respect to the current Q-function.

\paragraph{Trust region adaptation of target policy.} Motivated by the connection between off-policy Q($\lambda$) and trust region updates, we consider a variant of Q($\lambda$) which constructs the target distribution as the mixture between greedy and behavior policy (Eqn~\eqref{eq:mixing-target}). This introduces the mixing coefficient as an extra hyper-parameter to the algorithm $\alpha$, which we find to work best when it is around $0.6\sim 0.8$. Meanwhile, we  find $\alpha=1$ (i.e., greedy policy) to work generally sub-optimally.

In Figure~\ref{fig:deeprl}, we compare such a variant of off-policy distributional Q($\lambda$) against baseline C51 and Retrace-C51. In this case, distributional Q($\lambda$) obtains certain performance improvements over Retrace-C51, which further improves over C51 as shown in \citep{tang2022nature}.

\section{Conclusion}

We have proposed off-policy distributional Q($\lambda$), a new addition to the distributional RL arsenal. Without importance sampling distributional Q($\lambda$) has a few intriguing theoretical properties: its efficacy depends on the level of off-policyness and introduces unique interplay with signed measure representations. These properties set distributional Q($\lambda$) aside from previous approaches. Distributional Q($\lambda$) also enjoys promising empirical performance, when tested on both tabular and deep RL domains.

\bibliography{main}

\begin{thebibliography}{26}
\providecommand{\natexlab}[1]{#1}
\providecommand{\url}[1]{\texttt{#1}}
\expandafter\ifx\csname urlstyle\endcsname\relax
  \providecommand{\doi}[1]{doi: #1}\else
  \providecommand{\doi}{doi: \begingroup \urlstyle{rm}\Url}\fi

\bibitem[Agarwal et~al.(2021)Agarwal, Schwarzer, Castro, Courville, and
  Bellemare]{agarwal2021deep}
Rishabh Agarwal, Max Schwarzer, Pablo~Samuel Castro, Aaron~C Courville, and
  Marc Bellemare.
\newblock Deep reinforcement learning at the edge of the statistical precipice.
\newblock \emph{Advances in neural information processing systems},
  34:\penalty0 29304--29320, 2021.

\bibitem[Bellemare et~al.(2017{\natexlab{a}})Bellemare, Dabney, and
  Munos]{bellemare2017distributional}
Marc~G. Bellemare, Will Dabney, and R{\'e}mi Munos.
\newblock A distributional perspective on reinforcement learning.
\newblock In \emph{Proceedings of the International Conference on Machine
  Learning}, 2017{\natexlab{a}}.

\bibitem[Bellemare et~al.(2017{\natexlab{b}})Bellemare, Danihelka, Dabney,
  Mohamed, Lakshminarayanan, Hoyer, and Munos]{bellemare2017cramer}
Marc~G. Bellemare, Ivo Danihelka, Will Dabney, Shakir Mohamed, Balaji
  Lakshminarayanan, Stephan Hoyer, and R{\'e}mi Munos.
\newblock The {C}ramer distance as a solution to biased {W}asserstein
  gradients.
\newblock \emph{arXiv preprint arXiv:1705.10743}, 2017{\natexlab{b}}.

\bibitem[Bellemare et~al.(2019)Bellemare, Le~Roux, Castro, and
  Moitra]{bellemare2019distributional}
Marc~G Bellemare, Nicolas Le~Roux, Pablo~Samuel Castro, and Subhodeep Moitra.
\newblock Distributional reinforcement learning with linear function
  approximation.
\newblock In \emph{The 22nd International Conference on Artificial Intelligence
  and Statistics}, pages 2203--2211. PMLR, 2019.

\bibitem[Bellemare et~al.(2023)Bellemare, Dabney, and Rowland]{bdr2022}
Marc~G. Bellemare, Will Dabney, and Mark Rowland.
\newblock \emph{Distributional Reinforcement Learning}.
\newblock MIT Press, 2023.
\newblock \url{http://www.distributional-rl.org}.

\bibitem[Espeholt et~al.(2018)Espeholt, Soyer, Munos, Simonyan, Mnih, Ward,
  Doron, Firoiu, Harley, Dunning, et~al.]{espeholt2018impala}
Lasse Espeholt, Hubert Soyer, Remi Munos, Karen Simonyan, Volodymir Mnih, Tom
  Ward, Yotam Doron, Vlad Firoiu, Tim Harley, Iain Dunning, et~al.
\newblock Impala: Scalable distributed deep-rl with importance weighted
  actor-learner architectures.
\newblock \emph{arXiv preprint arXiv:1802.01561}, 2018.

\bibitem[Gruslys et~al.(2018)Gruslys, Dabney, Azar, Piot, Bellemare, and
  Munos]{gruslys2017reactor}
Audrunas Gruslys, Will Dabney, Mohammad~Gheshlaghi Azar, Bilal Piot, Marc~G.
  Bellemare, and R{\'e}mi Munos.
\newblock The {R}eactor: A fast and sample-efficient actor-critic agent for
  reinforcement learning.
\newblock In \emph{Proceedings of the International Conference on Learning
  Representations}, 2018.

\bibitem[Harutyunyan et~al.(2016)Harutyunyan, Bellemare, Stepleton, and
  Munos]{harutyunyan2016q}
Anna Harutyunyan, Marc~G. Bellemare, Tom Stepleton, and R{\'e}mi Munos.
\newblock Q($\lambda$) with off-policy corrections.
\newblock In \emph{Proceedings of the International Conference on Algorithmic
  Learning Theory}, 2016.

\bibitem[Kakade and Langford(2002)]{kakade2002approximately}
Sham Kakade and John Langford.
\newblock Approximately optimal approximate reinforcement learning.
\newblock In \emph{ICML}, volume~2, pages 267--274, 2002.

\bibitem[Kingma and Ba(2015)]{kingma2014adam}
Diederik~P. Kingma and Jimmy Ba.
\newblock Adam: A method for stochastic optimization.
\newblock In \emph{Proceedings of the International Conference on Learning
  Representations}, 2015.

\bibitem[Kozuno et~al.(2021)Kozuno, Tang, Rowland, Munos, Kapturowski, Dabney,
  Valko, and Abel]{kozuno2021revisiting}
Tadashi Kozuno, Yunhao Tang, Mark Rowland, R{\'e}mi Munos, Steven Kapturowski,
  Will Dabney, Michal Valko, and David Abel.
\newblock Revisiting {P}eng's {Q}($\lambda$) for modern reinforcement learning.
\newblock In \emph{Proceedings of the International Conference on Machine
  Learning}, 2021.

\bibitem[Mnih et~al.(2013)Mnih, Kavukcuoglu, Silver, Graves, Antonoglou,
  Wierstra, and Riedmiller]{mnih2013}
Volodymyr Mnih, Koray Kavukcuoglu, David Silver, Alex Graves, Ioannis
  Antonoglou, Daan Wierstra, and Martin Riedmiller.
\newblock Playing atari with deep reinforcement learning.
\newblock \emph{arXiv preprint arXiv:1312.5602}, 2013.

\bibitem[Mnih et~al.(2015)Mnih, Kavukcuoglu, Silver, Rusu, Veness, Bellemare,
  Graves, Riedmiller, Fidjeland, Ostrovski, Petersen, Beattie, Sadik,
  Antonoglou, King, Kumaran, Wierstra, Legg, and Hassabis]{mnih2015humanlevel}
Volodymyr Mnih, Koray Kavukcuoglu, David Silver, Andrei~A. Rusu, Joel Veness,
  Marc~G. Bellemare, Alex Graves, Martin Riedmiller, Andreas~K. Fidjeland,
  Georg Ostrovski, Stig Petersen, Charles Beattie, Amir Sadik, Ioannis
  Antonoglou, Helen King, Dharshan Kumaran, Daan Wierstra, Shane Legg, and
  Demis Hassabis.
\newblock Human-level control through deep reinforcement learning.
\newblock \emph{Nature}, 518\penalty0 (7540):\penalty0 529--533, February 2015.

\bibitem[Morimura et~al.(2010{\natexlab{a}})Morimura, Sugiyama, Kashima,
  Hachiya, and Tanaka]{morimura2010nonparametric}
Tetsuro Morimura, Masashi Sugiyama, Hisashi Kashima, Hirotaka Hachiya, and
  Toshiyuki Tanaka.
\newblock Nonparametric return distribution approximation for reinforcement
  learning.
\newblock In \emph{Proceedings of the International Conference on Machine
  Learning}, 2010{\natexlab{a}}.

\bibitem[Morimura et~al.(2010{\natexlab{b}})Morimura, Sugiyama, Kashima,
  Hachiya, and Tanaka]{morimura2012parametric}
Tetsuro Morimura, Masashi Sugiyama, Hisashi Kashima, Hirotaka Hachiya, and
  Toshiyuki Tanaka.
\newblock Parametric return density estimation for reinforcement learning.
\newblock In \emph{Proceedings of the Conference on Uncertainty in Artificial
  Intelligence}, 2010{\natexlab{b}}.

\bibitem[Munos et~al.(2016)Munos, Stepleton, Harutyunyan, and
  Bellemare]{munos2016safe}
R{\'e}mi Munos, Tom Stepleton, Anna Harutyunyan, and Marc~G. Bellemare.
\newblock Safe and efficient off-policy reinforcement learning.
\newblock In \emph{Advances in Neural Information Processing Systems}, 2016.

\bibitem[Nam et~al.(2021)Nam, Kim, and Park]{nam2021gmac}
Daniel~W. Nam, Younghoon Kim, and Chan~Y. Park.
\newblock {GMAC}: A distributional perspective on actor-critic framework.
\newblock In \emph{Proceedings of the International Conference on Machine
  Learning}, 2021.

\bibitem[Peng and Williams(1994)]{peng1994incremental}
Jing Peng and Ronald~J Williams.
\newblock Incremental multi-step q-learning.
\newblock In \emph{Machine Learning Proceedings 1994}, pages 226--232.
  Elsevier, 1994.

\bibitem[Precup et~al.(2001)Precup, Sutton, and Dasgupta]{precup2001off}
Doina Precup, Richard~S. Sutton, and Sanjoy Dasgupta.
\newblock Off-policy temporal-difference learning with function approximation.
\newblock In \emph{Proceedings of the International Conference on Machine
  Learning}, 2001.

\bibitem[Rowland et~al.(2018)Rowland, Bellemare, Dabney, Munos, and
  Teh]{rowland2018analysis}
Mark Rowland, Marc~G. Bellemare, Will Dabney, R{\'e}mi Munos, and Yee~Whye Teh.
\newblock An analysis of categorical distributional reinforcement learning.
\newblock In \emph{Proceedings of the International Conference on Artificial
  Intelligence and Statistics}, 2018.

\bibitem[Rowland et~al.(2020)Rowland, Dabney, and Munos]{rowland2019adaptive}
Mark Rowland, Will Dabney, and R{\'e}mi Munos.
\newblock Adaptive trade-offs in off-policy learning.
\newblock In \emph{Proceedings of the International Conference on Artificial
  Intelligence and Statistics}, 2020.

\bibitem[Schulman et~al.(2015)Schulman, Levine, Abbeel, Jordan, and
  Moritz]{schulman2015}
John Schulman, Sergey Levine, Pieter Abbeel, Michael Jordan, and Philipp
  Moritz.
\newblock Trust region policy optimization.
\newblock In \emph{International Conference on Machine Learning}, pages
  1889--1897, 2015.

\bibitem[Sutton(1988)]{sutton1988learning}
Richard~S. Sutton.
\newblock Learning to predict by the methods of temporal differences.
\newblock \emph{Machine learning}, 3\penalty0 (1):\penalty0 9--44, 1988.

\bibitem[Sutton and Barto(1998)]{sutton1998}
Richard~S. Sutton and Andrew~G. Barto.
\newblock \emph{Reinforcement Learning: An Introduction}.
\newblock MIT Press, 1998.

\bibitem[Tang et~al.(2020)Tang, Valko, and Munos]{tang2020taylor}
Yunhao Tang, Michal Valko, and R{\'e}mi Munos.
\newblock Taylor expansion policy optimization.
\newblock \emph{arXiv preprint arXiv:2003.06259}, 2020.

\bibitem[Tang et~al.(2022)Tang, Munos, Rowland, Pires, Dabney, and
  Bellemare]{tang2022nature}
Yunhao Tang, Remi Munos, Mark Rowland, Bernardo~Avila Pires, Will Dabney, and
  Marc~G Bellemare.
\newblock The nature of temporal difference errors in multi-step distributional
  reinforcement learning.
\newblock In Alice~H. Oh, Alekh Agarwal, Danielle Belgrave, and Kyunghyun Cho,
  editors, \emph{Advances in Neural Information Processing Systems}, 2022.
\newblock URL \url{https://openreview.net/forum?id=Mn4IkuWamy}.

\end{thebibliography}
\bibliographystyle{plainnat}

\clearpage
\onecolumn

\begin{appendix}

\section*{\centering APPENDICES: Off-policy Q($\lambda$) for distributional reinforcement learning}

\section{Detailed derivations of off-policy distributional Q($\lambda$)}
\label{appendix:alternative}

Here, we provide a detailed alternative derivation of off-policy distributional Q($\lambda$) operator. We start with the on-policy $n$-step distributional operator
\begin{align*}
    \mathcal{T}_n^\pi\eta(x,a) = \mathbb{E}_\pi\left[\left(\textrm{b}_{G_{0:n-1},\gamma^{n}}\right)_{\#} \eta(X_n,A_n^\pi)  \right],
\end{align*}
which reduces to the distributional Bellman operator $\mathcal{T}^\pi$ when $n=1$ \citep{bellemare2017distributional}. Note the $n$-step operator has $\eta^\pi$ as the unique fixed point.

The on-policy distributional Q($\lambda$) operator can be constructed as the geometrically weighted mixture of $n$-step distributional Bellman operator.
\begin{align*}
    \mathcal{T}_\lambda^{\pi}\eta(x,a) &\coloneqq (1-\lambda) \sum_{n=1}^\infty \lambda^{n-1} \mathcal{T}_n^\pi\eta(x,a) \\ 
    &= (1-\lambda)\sum_{n=1}^\infty \mathbb{E}_\pi\left[\left(\textrm{b}_{G_{0:n-1},\gamma^{n}}\right)_{\#} \eta(X_n,A_n^\pi)  \right].
\end{align*}
The on-policy Q($\lambda$) operator has $\eta^\pi$ as the unique fixed point by design. The on-policy nature of the operator is reflected by the fact that the expectation is taken under target policy $\pi$. Now, we rewrite the above operator in the form of distributional TD error,
\begin{align*}
   \mathcal{T}_\lambda^{\pi}\eta(x,a)= \eta(x,a) + \mathbb{E}_{\pi}\left[ \sum_{t=0}^\infty \lambda^t \cdot \left(\textrm{b}_{G_{0:t-1},\gamma^{t}}\right)_{\#} \Delta_t^\pi  \right],
\end{align*}
where $\Delta_t^\pi=\mathcal{T}^\pi\eta(X_t,A_t)-\eta(X_t,A_t)$ is a signed measure with zero total mass. To derive the off-policy distributional Q($\lambda$) operator, we simply replace the expectation under $\pi$ by an expectation under behavior policy $\mu$. This yields the operator
\begin{align*}
   \mathcal{A}_\lambda^{\pi,\mu}\eta(x,a)= \eta(x,a) + \mathbb{E}_{\mu}\left[ \sum_{t=0}^\infty \lambda^t \cdot \left(\textrm{b}_{G_{0:t-1},\gamma^{t}}\right)_{\#} \Delta_t^\pi  \right].
\end{align*}

\section{Distributional Peng's Q($\lambda$) operator}
\label{appendix:peng}
We provide a more detailed discussion on distributional Peng's Q($\lambda$) operator. Starting from the on-policy $n$-step distributional operator,
\begin{align*}
    \mathcal{T}_n^\pi\eta(x,a) = \mathbb{E}_\pi\left[\left(\textrm{b}_{G_{0:n-1},\gamma^{n}}\right)_{\#} \eta(X_n,A_n^\pi)  \right],
\end{align*}
we derive the uncorrected $n$-step operator, by simply replacing the expectation under $\pi$ by an expectation under $\mu$,
\begin{align*}
    \mathcal{P}_n^{\pi,\mu}\eta(x,a) = \mathbb{E}_\mu\left[\left(\textrm{b}_{G_{0:n-1},\gamma^{n}}\right)_{\#} \eta(X_n,A_n^\pi)  \right].
\end{align*}
The uncorrected $n$-step operator, as its name suggests, does not have $\eta^\pi$ as the fixed point in general. This is because the operator does not correct for the off-policyness between $\pi$ and $\mu$ and takes a plain expectation over $\mu$. The  distributional Peng's Q($\lambda$) operator, is simply a geometrically weighted mixture of uncorrected $n$-step operators
\begin{align*}
    \mathcal{P}_\lambda^{\pi,\mu}\eta(x,a) &\coloneqq (1-\lambda) \sum_{n=1}^\infty \lambda^{n-1} \mathcal{P}_n^{\pi,\mu}\eta(x,a) \\ 
    &= (1-\lambda)\sum_{n=1}^\infty \mathbb{E}_\mu\left[\left(\textrm{b}_{G_{0:n-1},\gamma^{n}}\right)_{\#} \eta(X_n,A_n^\pi)  \right].
\end{align*}

\section{Proof of theoretical results}
\label{appendix:proof}

\lemmasigned*
\begin{proof}
   Our proof follows closely the proof techniques of Lemma 3.1 in Tang et al \citep{tang2022nature}. Following their approach, we consider the general notation of trace coefficient $c_t$ which in our case is $\lambda$. For all $t\geq 1$, we define the coefficient
\begin{align*}
    w_{y,b,r_{0:t-1}} \coloneqq \mathbb{E}_\mu\left[c_1...c_{t-1}\left(\pi(b|X_t) - c(X_t,b)\mu(b|X_t)\right)\cdot \mathbb{I}[X_t=y]\Pi_{s=0}^{t-1}\mathbb{I}[R_s=r_s]\right].
\end{align*}
Let $\mathscr{R}$ be the set of reward value that random variable $R_t$ can take. Let $\mathcal{R}^t=\mathcal{R}\times \mathcal{R}\times...\mathcal{R}$ be the Cartesian product of $t$ replicates of $\mathcal{R}$. Through careful algebra, we can rewrite the off-policy Q($\lambda$) operator as follows
\begin{align*}
    \mathcal{A}_\lambda^{\pi,\mu}\eta(x,a) = \sum_{t=1}^{\infty} \sum_{y\in\mathcal{X}} \sum_{b\in\mathcal{A}} \sum_{r_{0:t-1}\in\mathcal{R}^t} w_{y,b,r_{0:t-1}} \left(\textrm{b}_{G_{0:t-1},\gamma^t}\right)_{\#} \eta(y,b).
\end{align*}
Note that each term of the form $\left(\textrm{b}_{G_{0:t-1},\gamma^t}\right)_{\#} \eta(y,b)$ corresponds to applying a pushforward operation $\left(\textrm{b}_{G_{0:t-1},\gamma^t}\right)_{\#}$ on the signed measure $\eta(x,a)$, which means $\left(\textrm{b}_{G_{0:t-1},\gamma^t}\right)_{\#} \eta(y,b)\in\signedspace_\infty(\mathbb{R})$. Now, we examine the sum of all coefficients $\sum w_{y,b,r_{0:t-1}} =  \sum_{t=1}^{\infty} \sum_{x\in\mathcal{X}}\sum_{b\in\mathcal{A}}\sum_{r_{0:t-1}\in\mathcal{R}^t} w_{y,b,r_{0:t-1}}$. Tang et al. has showed that for general $c_t$,
\begin{align*}
  \sum w_{y,b,r_{0:t-1}}
   = 1
\end{align*}
This implies that $\mathcal{A}_\lambda^{\pi,\mu}\in\signedspace_1(\mathbb{R})$ as it is a linear combination of signed measures with total mass $1$. A critical difference here is that since $c_t\not\in[0,\rho_t]$ as in the Retrace case, there is no general guarantee that $w_{y,b,r_{0:t-1}}\geq 0$.
\end{proof}

\lemmaqlambdafixedpoint

\begin{proof}
By construction, $\mathcal{A}_\lambda^{\pi,\mu}$ is a weighted sum of distributional TD error $\Delta_t^\pi$ under policy $\pi$. By letting $\eta=\eta^\pi$, we have $\mathbb{E}\left[\Delta_t^\pi\;\middle|\;X_t,A_t\right]=0$ (note that here the right hand side is a \emph{zero measure}). This implies $\mathcal{A}_\lambda^{\pi,\mu}\eta^\pi=\eta^\pi$ and verifies $\eta^\pi$ as a fixed point of the operator.
\end{proof}

\lemmaqlambdacontraction*

\begin{proof}
From the proof of Lemma~\ref{lemma:signed}, we can write
\begin{align*}
    \mathcal{A}_\lambda^{\pi,\mu}\eta(x,a) = \sum_{t=1}^{\infty} \sum_{x\in\mathcal{X}} \sum_{a\in\mathcal{A}} \sum_{r_{0:t-1}\in\mathcal{R}^t} w_{x,a,r_{0:t-1}} \left(\textrm{b}_{G_{0:t-1},\gamma^t}\right)_{\#} \eta(x,a).
\end{align*}
For $1\leq t\leq n$, we have
\begin{align*}
    w_{x,a,r_{0:t-1}} &\coloneqq \mathbb{E}_\mu\left[c_1...c_{t-1} \left(\pi(a|X_{t})-c(X_t,a)\mu(a|X_t)\right) \cdot \mathbb{I}[X_{t}=x]\Pi_{s=0}^{t-1}\mathbb{I}[R_s=r_s]\right] \\
    &= \mathbb{E}_\mu\left[\lambda^{t-1} \left(\pi(a|X_{t})-\lambda\mu(a|X_t)\right) \cdot \mathbb{I}[X_{t}=x]\Pi_{s=0}^{t-1}\mathbb{I}[R_s=r_s]\right]
\end{align*}
For any $t\geq 1$, we upper bound the absolute value of the weight coefficient $w_{x,a,r_{0:t-1}}$ as follows
\begin{align*}
    &= \left| \lambda^{t-1}\cdot (1-\lambda)\mathbb{E}\left[\pi(a|X_t)\cdot \mathbb{I}[X_{t}=x]\Pi_{s=0}^{t-1}\mathbb{I}[R_s=r_s]\right] + \lambda^{t-1}\lambda \cdot  \mathbb{E}\left[ \left(\pi(a|X_t)-\mu(a|X_t)\right) \mathbb{I}[X_{t}=x]\Pi_{s=0}^{t-1}\mathbb{I}[R_s=r_s]\right] \right|
    \\ &\leq_{(a)} \lambda^{t-1}\cdot (1-\lambda)\mathbb{E}\left[\pi(a|X_t)\cdot \mathbb{I}[X_{t}=x]\Pi_{s=0}^{t-1}\mathbb{I}[R_s=r_s]\right] + \lambda^{t-1}\lambda \cdot \mathbb{E}\left[\left|\pi(a|X_t)-\mu(a|X_t)\right| \cdot \mathbb{I}[X_{t}=x]\Pi_{s=0}^{t-1}\mathbb{I}[R_s=r_s]\right]  \\
    &\leq_{(b)} \lambda^{t-1}\cdot (1-\lambda)\mathbb{E}\left[\pi(a|X_t)\cdot \mathbb{I}[X_{t}=x]\Pi_{s=0}^{t-1}\mathbb{I}[R_s=r_s]\right] + \lambda^{t-1}\lambda \cdot \mathbb{E}\left[\epsilon \cdot \mathbb{I}[X_{t}=x]\Pi_{s=0}^{t-1}\mathbb{I}[R_s=r_s]\right] 
    \eqqcolon |w_{x,a,r_{0:t-1}}|
\end{align*}
Above, (a) follows from the triangle inequality; (b) follows from the fact that for any random variable $Z$,  $|\mathbb{E}[Z]|\leq \mathbb{E}[|Z|]$; for (b), we also apply the fact that $\left\lVert \pi-\mu\right\rVert_1\coloneqq \max_{x\in\mathcal{X}}\sum_{a\in\mathcal{A}}|\pi(a|x)-\mu(a|x)|=\epsilon$. 
Now, we define a signed measure as the negative of the distribution $\left(\textrm{b}_{G_{0:t-1},\gamma^t}\right)_{\#} \eta(x,a)$
\begin{align*}
   \tilde{\eta}_{x,a,r_{0:t-1}} \coloneqq \text{sign}(w_{x,a,r_{0:t-1}}) \cdot \left(\textrm{b}_{G_{0:t-1},\gamma^t}\right)_{\#} \eta(x,a),
\end{align*}
with the signed function $\text{sign}(z):\mathbb{R}\rightarrow\mathbb{R}$. Then we can write
\begin{align*}
    \mathcal{A}_\lambda^{\pi,\mu}\eta(x,a) = \sum_{t=1}^{\infty} \sum_{x\in\mathcal{X}} \sum_{a\in\mathcal{A}} \sum_{r_{0:t-1}\in\mathcal{R}^t} |w_{x,a,r_{0:t-1}}| \tilde{\eta}_{x,a,r_{0:t-1}}.
\end{align*}
Finally, we have
\begin{align*}
   &\ell_p^p\left(\mathcal{A}_\lambda^{\pi,\mu}\eta_1(x,a), \mathcal{A}_\lambda^{\pi,\mu}\eta_2(x,a)\right) \\
   &=_{(a)} \ell_p^p\left(\sum_{t=1}^{\infty} \sum_{x\in\mathcal{X}} \sum_{a\in\mathcal{A}} \sum_{r_{0:t-1}\in\mathcal{R}^t} |w_{x,a,r_{0:t-1}}| \tilde{\eta}_{x,a,r_{0:t-1}}^{(1)},\sum_{t=1}^{\infty} \sum_{x\in\mathcal{X}} \sum_{a\in\mathcal{A}} \sum_{r_{0:t-1}\in\mathcal{R}^t} |w_{x,a,r_{0:t-1}}| \tilde{\eta}_{x,a,r_{0:t-1}}^{(2)}\right) \\
    &=_{(b)} \left(\sum_{x\in\mathcal{X}} \sum_{a\in\mathcal{A}} \sum_{r_{0:t-1}\in\mathcal{R}^t} |w_{x,a,r_{0:t-1}}|\right)^{p-1} \sum_{x\in\mathcal{X}} \sum_{a\in\mathcal{A}} \sum_{r_{0:t-1}\in\mathcal{R}^t} |w_{x,a,r_{0:t-1}}| \ell_p^p\left(\tilde{\eta}_{x,a,G_{0:t-1}}^{(1)},\tilde{\eta}_{x,a,G_{0:t-1}}^{(2)}\right) \\
    &\leq_{(c)} \left(\sum_{t=1}^{\infty}\sum_{x\in\mathcal{X}} \sum_{a\in\mathcal{A}} \sum_{r_{0:t-1}\in\mathcal{R}^t} |w_{x,a,r_{0:t-1}}| \right)^{p-1} \sum_{x\in\mathcal{X}} \sum_{a\in\mathcal{A}} \sum_{r_{0:t-1}\in\mathcal{R}^t} |w_{x,a,r_{0:t-1}}| \ell_p^p\left(\tilde{\eta}_{x,a,r_{0:t-1}^{(1)}},\tilde{\eta}_{x,a,r_{0:t-1}}^{(2)}\right) \\
    &\leq_{(d)} \left(\sum_{t=1}^{\infty}\sum_{x\in\mathcal{X}} \sum_{a\in\mathcal{A}} \sum_{r_{0:t-1}\in\mathcal{R}^t} |w_{x,a,r_{0:t-1}}| \right)^{p-1} \sum_{t=1}^{\infty}\sum_{x\in\mathcal{X}} \sum_{a\in\mathcal{A}} \sum_{r_{0:t-1}\in\mathcal{R}^t} |w_{x,a,r_{0:t-1}}| \gamma^t\bar{\ell}_p^p\left(\tilde{\eta}_1,\tilde{\eta}_2\right) \\
    &=_{(e)} \left(\sum_{t=1}^{\infty}\sum_{x\in\mathcal{X}} \sum_{a\in\mathcal{A}} \sum_{r_{0:t-1}\in\mathcal{R}^t} |w_{x,a,r_{0:t-1}}| \right)^{p-1} \sum_{t=1}^{\infty}\sum_{x\in\mathcal{X}} \sum_{a\in\mathcal{A}} \sum_{r_{0:t-1}\in\mathcal{R}^t} |w_{x,a,r_{0:t-1}}| \gamma^t\bar{\ell}_p^p\left(\eta_1,\eta_2\right).
\end{align*}
In the above, (a) follows from the definition of $\tilde{\eta}$; (b) follows from the scaling property of the $\ell_p$ distance; (c) follows from the convex property of the $\ell_p$ distance, see Bellemare et al. \citep{bdr2022}; (d) follows from the definition of the supremum distance $\bar{\ell}_p$; (e) follows from the fact that $\bar{\ell}_p(\eta_1,\eta_2)=\bar{L}(\tilde{\eta}_1,\tilde{\eta}_2)$. Now, we examine the sum over coefficients $|w_{x,a,r_{0:t-1}}|$. For any fixed time step $t\geq 1$,
\begin{align*}
    \sum_{x\in\mathcal{X}} \sum_{a\in\mathcal{A}} \sum_{r_{0:t-1}\in\mathcal{R}^t} |w_{x,a,r_{0:t-1}}|=\lambda^{t-1}(1-\lambda+\lambda\epsilon).
\end{align*}
Hence the total sum is 
\begin{align*}
    \sum_{t=1}^{\infty}\sum_{x\in\mathcal{X}} \sum_{a\in\mathcal{A}} \sum_{r_{0:t-1}\in\mathcal{R}^t} |w_{x,a,r_{0:t-1}}| &= \frac{1-\lambda+\lambda\epsilon}{1-\lambda}\\
    \sum_{t=1}^{\infty}\sum_{x\in\mathcal{X}} \sum_{a\in\mathcal{A}} \sum_{r_{0:t-1}\in\mathcal{R}^t} |w_{x,a,r_{0:t-1}}| \bar{\ell}_p^p(\eta_1,\eta_2) &= \frac{\gamma(1-\lambda+\lambda\epsilon)}{1-\lambda\gamma} \bar{\ell}_p^p(\eta_1,\eta_2).
\end{align*}
By combing hte above quantities and taking the $1/p$-th root, we obtain the overall result
\begin{align*}
    &\ell_p\left(\mathcal{A}_\lambda^{\pi,\mu}\eta_1(x,a), \mathcal{A}_\lambda^{\pi,\mu}\eta_2(x,a)\right) \leq \beta_p(x,a) \bar{\ell}_p(\eta_1,\eta_2)
\end{align*}
with $\beta_p(x,a)=\gamma^{1/p}\frac{1-\lambda+\lambda\epsilon}{(1-\lambda)^{(p-1)/p}(1-\lambda\gamma)^{1/p}}$. We obtain the overall contraction rate by taking $\beta_p=\max_{x,a}\beta_p(x,a)$.
\end{proof}

\corollarycontraction*

\begin{proof}
By setting the condition on the contraction rate $\beta_1<1$, we obtain $\left\lVert \pi-\mu\right\rVert_1<\frac{1-\gamma}{\lambda\gamma}$. Since $\eta^\pi$ is a fixed point of $\mathcal{A}_\lambda^{\pi,\mu}$ by Lemma~\ref{lemma:qlambdafixedpoint}, when the operator is contractive $\beta_1<1$ under the $L_1$ distance, we also have $\eta^\pi$ as the unique fixed point.
\end{proof}

\lemmaprojectedqlambdacontraction*
\begin{proof}
Since $\mathcal{A}_\lambda^{\pi,\mu}$ is $\beta_2$-contractive under the $\bar{L}_2$ distance and the categorical projection $\Pi_c$ is non-expansive under the $\bar{L}_2$ distance \citep{bellemare2019distributional}, it follows that the composed operator $\Pi_c\mathcal{A}_\lambda^{\pi,\mu}$ is also $\beta_2$-contractive.
\end{proof}

\lemmaprojectedqlambdaerror*
\begin{proof}
The result follows from an application of Proposition 5.28 in Bellemare et al. \citep{bdr2022} to the off-policy distributional Q($\lambda$) case.

\end{proof}

\section{Alternative way to construct  distributional Q($\lambda$)}
\label{appendix:alternative}

The off-policy distributional Q($\lambda$) depends on the \emph{path dependent} multi-step TD error $\left(\textrm{b}_{G_{0:t-1},\gamma^{t}}\right)_{\#} \Delta_t^\pi$. Here, the path-dependency stems from the fact that this TD error depends on the path of reward $R_{0:t}$ \citep{tang2022nature} and is fundamentally different from value-based TD error $\delta_t^\pi$, which depends on the one-step transition $(X_t,A_t,R_t)$ only.

Nevertheless, we can build an alternative variant of distributional Q($\lambda$) by removing the path-dependency. The derivation might look heuristic: while the transformation $\left(\textrm{b}_{G_{0:t-1},\gamma^{t}}\right)_{\#} \Delta_t^\pi$ shrinks the width of the signed measure $\Delta_t^\pi$ by a factor $\gamma^t$, we can instead shrink its height by pulling the factor outside of the pushforward, leading to $\gamma^t (\textrm{b}_{G_{0:t-1},1})_{\#}\Delta_t^\pi$. This produces the operator
\begin{align*}
    \tilde{\mathcal{A}}_\lambda^{\pi,\mu}\eta(x,a) \coloneqq \eta(x,a) + \mathbb{E}_{\mu}\left[ \sum_{t=0}^\infty \textcolor{blue}{\gamma^t\lambda^t} \cdot \left(\textrm{b}_{G_{0:t-1},1}\right)_{\#} \Delta_t^\pi  \right].
\end{align*}
While the derivation is quite technical, we note that this  formulation can be understood as treating the discount factor $\gamma$ as a termination probability, rather than a scaling factor that defines the cumulative return. This is related to an alternative interpretation of the random return that recovers the same expectation as $\sum_{t=0}^\infty\gamma^tR_t$, see \citet{bdr2022} Chapter 2 for more discussions.

By design the operator also has $\eta^\pi$ as its fixed point. However, as with $\mathcal{A}_\lambda^{\pi,\mu}$, its contraction property depends on $\lambda$.

\begin{restatable}{lemma}{lemmacontractionalternative}\label{lemma:contractionalternative} (\textbf{Contraction of alternative operator}) The alternative operator satisfies the following property: for any $\eta_1,\eta_2\in\mathcal{M}_1(\mathbb{R})^{\mathcal{X}\times\mathcal{A}}$,
\begin{align*}
    \bar{\ell}_p\left( \tilde{\mathcal{A}}_\lambda^{\pi,\mu}\eta_1, \tilde{\mathcal{A}}_\lambda^{\pi,\mu}\eta_2\right) \leq \tilde{\beta} \bar{\ell}_p(\eta_1,\eta_2),
\end{align*}
with $\tilde{\beta}=\frac{\gamma(1+\lambda)}{1-\gamma\lambda}$. When $\lambda<\frac{1-\gamma}{2\gamma}$, we have $\tilde{\beta}<1$ and the operator is guaranteed to be contractive.
\end{restatable}

\begin{proof}
The proof idea is similar to Lemma~\ref{lemma:qlambdacontraction}, where we seek to write the back-up target as a convex combination of signed measudres. Indeed, for any $\eta\in\mathcal{M}_1(\mathbb{R})^{\mathcal{X}\times\mathcal{A}}, $we can write
\begin{align*}
    \tilde{\mathcal{A}}_\lambda^{\pi,\mu}\eta(x,a) = \sum_{t=1}^{\infty} \sum_{x\in\mathcal{X}} \sum_{a\in\mathcal{A}} \sum_{r_{0:t-1}\in\mathcal{R}^t} \tilde{w}_{x,a,r_{0:t-1}} \tilde{\eta}_{x,a,r_{0:t-1}}.
\end{align*}
where $\tilde{w}$ can be negative, as before. The individual distribution writes from the pushforward operation that defines the operator
\begin{align*}
   \tilde{\eta}_{x,a,r_{0:t-1}} \coloneqq  \gamma^t \left(\textrm{b}_{G_{0:t-1},1}\right)_{\#} \eta(x,a).
\end{align*}
As a technical note, we see that unlike $\mathcal{A}_\lambda^{\pi,\mu}$, this operator cannot be written in a telescoping form. Hence we have the bound on the coefficient as
\begin{align*}
   \sum_{t=1}^\infty \sum_{x\in\mathcal{X}} \sum_{a\in\mathcal{A}} \sum_{r_{0:t-1}} |\tilde{w}_{x,a,r_{0:t-1}}| \leq \sum_{t=0}^\infty \gamma^t\lambda&t \cdot\gamma + \sum_{t=1}^\infty\gamma^t\lambda^t =\frac{\gamma(1+\lambda)}{1-\gamma\lambda}.
\end{align*}
This concludes the proof.

\end{proof}

We see that the upper bound on the trace coefficient $\lambda$ is $\frac{1-\gamma}{2\gamma}$ for the operator to be contractive, and strictly worse than the bound for the off-policy distributional Q($\lambda$) operator $\mathcal{A}_\lambda^{\pi,\mu}$ since $\left\lVert \pi-\mu\right\rVert_1\leq 2$. A direct implication is that $\tilde{A}_\lambda^{\pi,\mu}$ cannot benefit from multi-step learning even when on-policy, since the bound on $\lambda$ is constant: for $\gamma=0.99$, we have $\frac{1-\gamma}{2\gamma}\approx 0.01$, which is almost like one-step learning. The comparison suggests that $\mathcal{A}_\lambda^{\pi,\mu}$ is a much better construct of the multi-step learning operator.

\section{Optimal control} \label{appendix:optimal}

We can apply distributional Q($\lambda$) for optimal control. We use the notation $Q_{\eta_k}\in\mathbb{R}^{\mathcal{X}\times\mathcal{A}}$ to denote the Q-function induced by the signed return distribution $\eta_k$. At iteration $k$, we let the target policy to be the greedy policy with respect to $Q_{\eta_k}$, denoted as $\mathcal{G}\left(Q_{\eta_k}\right)$. Consider the recursion
\begin{align}
    \eta_{k+1} = \Pi_c \mathcal{A}^{\mathcal{G}\left(Q_{\eta_k}\right),\mu} \eta_k.\label{eq:optimal-control}
\end{align}
Under a few regularity conditions, we can guarantee that when $\lambda$ is small enough, the above recursion converges to a signed return distribution which closely approximates the return distribution $\eta^\ast\coloneqq \eta^{\pi^\ast}$ of the optimal policy $\eta^\ast$. 
\begin{restatable}{lemma}{lemmaoptimalqlambda}\label{lemma:optimalqlambda} (\textbf{Optimal control}) Assume the MDP has a unique deterministic optimal policy $\pi^\ast$. When $\lambda<\frac{1-\gamma}{2\gamma}$, we have $\eta_k\rightarrow\eta_\mathcal{A}^\ast\in\signedspace_1(\mathbb{R})^{\mathcal{X}\times\mathcal{A}}$ in $\bar{L}_2$ and 
\begin{align*}
    \bar{L}_2\left(\eta^\ast,\eta_\mathcal{A}^\ast\right) \leq \frac{\bar{L}_2\left(\eta^\ast,\Pi_c\eta^\ast\right)}{\sqrt{1-\gamma^2}}
\end{align*}
\end{restatable}
\begin{proof}
The proof is a combination of the proof techniques applied in categorical distributional Q-learning  \citep{rowland2018analysis} and value-based Q($\lambda$) \citep{harutyunyan2016q}.

Let $Q_k$ be the Q-function induced by the iterate $\eta_k$, then we can argue that the evolution of $Q_k$ is as if the Q-function iterates are generated under the value-based Q($\lambda$) algorithm. Note that the condition on $\lambda<\frac{1-\gamma}{2\gamma}$ means that $\mathcal{A}_\lambda^{\pi,\mu}$ is contractive for any $\pi,\mu$. Indeed $\left\lVert \pi-\mu \right\rVert_1<2$ and hence for any $\pi,\mu$, the corresponding bound on the contraction rate is less than $1$. This means the Q-function iterate $Q_k$ will converge to the optimal Q-function $Q^\ast$ by the unique optimal policy $\pi^\ast$.

Then we follow an identical trace of argument from \citet{rowland2018analysis} to show the convergence of the signed measure iterate $\eta_k$, this includes proving the existence of a limiting signed measure $\eta_\mathcal{A}^\ast$ and its $\bar{L}_2$ distance to the optimal return distribution $\eta^\ast$.

\end{proof}

Since the above result is inherited from the policy evaluation case, it puts a fairly conservative restriction on $\lambda$. In practice, we find that using a larger value of $\lambda$ can also lead to stable learning in both tabular and large-scale settings (Section~\ref{sec:exp}).

In Figure~\ref{fig:qlambda_control}, we compare the speed of convergence by measuring the Cramer distance $\ell_2\left(\eta^\ast(x_0,a_0)),\eta_k(x_0,a_0)\right)$ at a fixed state action pair $(x_0,a_0)$ throughout iterations and across randomly generated MDPs. 
When $\lambda$ is properly chosen (in this case $\lambda=0.4$), off-policy Q($\lambda$) improves over baselines both in terms of the rate of convergence and the asymptotic accuracy, compatible with observations made in the off-policy evaluation case. However, a caveat is that the improvement of Q($\lambda$) comes at the cost of having to tune trace parameter $\lambda$ in practice: when $\lambda$ is too small, the learning is guaranteed to be stable but one does not benefit from multi-step learning ($\lambda=0$ recovers the one-step algorithm); when $\lambda$ is too large (e.g., $\lambda\approx 1$), the learning can become unstable as shown in the experiments. 

\begin{figure}[t]
    \centering
    \includegraphics[keepaspectratio,width=.4\textwidth]{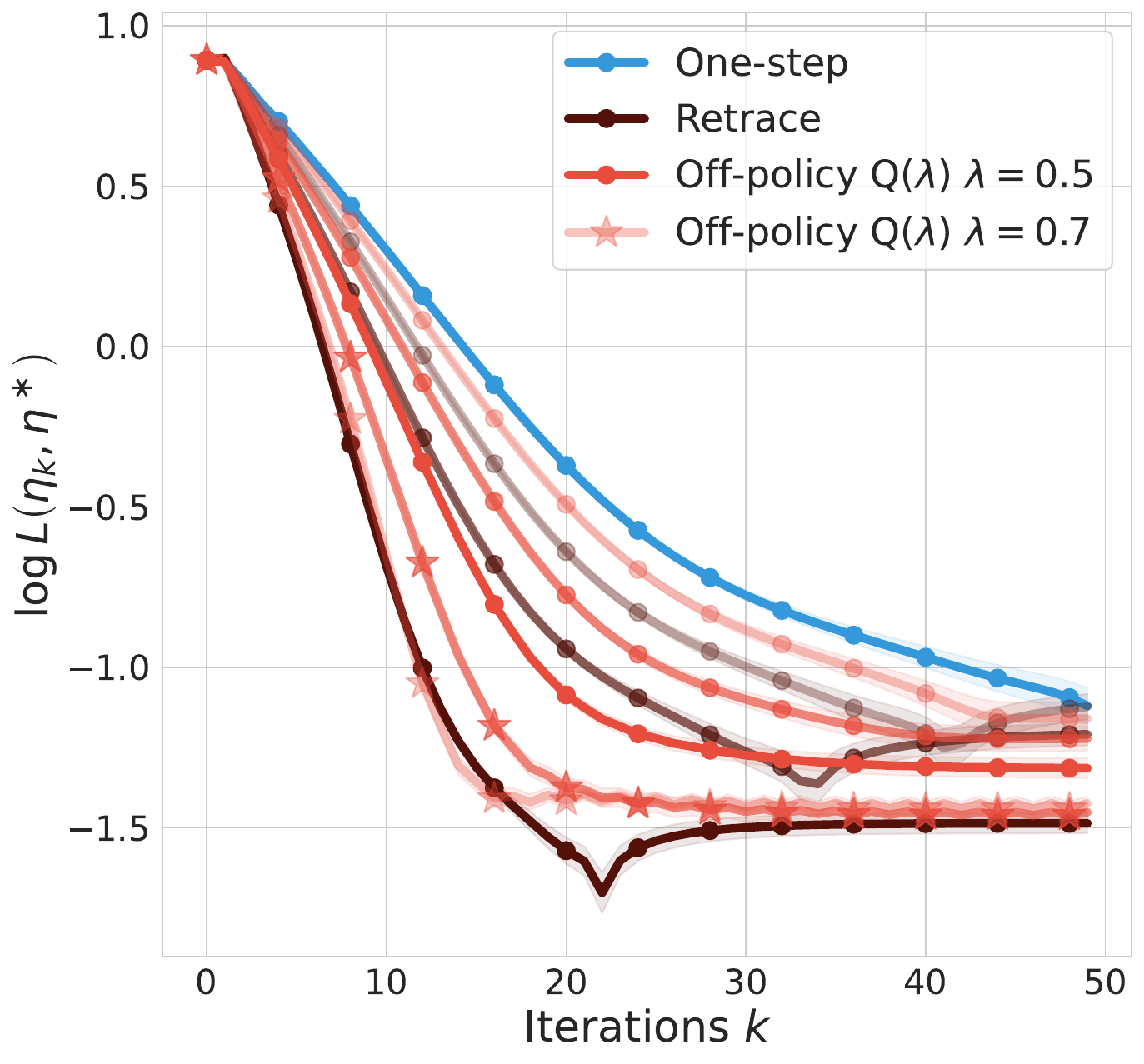}
    \caption{The distance between the algorithmic iterate $\eta_k$ and return distribution for the optimal policy $\eta^\ast$, as we run control algorithms with distributional one-step, Retrace and off-policy Q($\lambda$). All algorithms use categorical representations and set greedy policy as the target policy. Different curves show an algorithmic variant with a different hyper-parameter setting ($\bar{c}$ for Retrace and $\lambda$ for Q($\lambda$)). Unlike Figure~\ref{fig:qlambda_control} with $|\mathcal{A}|=20$, here with $|\mathcal{A}|=5$ all algorithmic behavior changes slightly. Since the problem effectively becomes less off-policy, Retrace can benefit from the full trace with $\bar{c}=4$, outperforming Q($\lambda$); meanwhile, Q($\lambda$) becomes more stable across all $\lambda$ values.}
    \label{fig:qlambda_control_5actions}
\end{figure}

\section{Experiments}
\label{appendix:exp}

We provide additional experimental details and results.

\subsection{Tabular MDP}

For the tabular MDP, the transition probability $p(\cdot|x,a)$ is randomly sampled from a Dirichlet distribution with $|\mathcal{X}|$ entries with a rate of $0.1$, i.e., $\text{Dirichlet}\left([0.1,0.1,...0.1]\right)$, for all $(x,a)$ independently. The reward function $r(x,a)$ is randomly sampled from $\mathcal{N}(0,1)$ and fixed as a deterministic reward for each $(x,a)$. The discount factor is set as $\gamma=0.9$. The behavior policy $\mu$ is uniform throughout, and target policy is always greedy with respect to the Q-function induced by the current distribution iterate. For all experiments, we generate the iterate as
\begin{align*}
    \eta_{k+1} = \Pi_c \mathcal{R}\eta_k
\end{align*}
where $\mathcal{R}$ is the distributional operator of interest. The initial iterate $\eta_0$ assigns uniform weights across all $m=10$ atoms. 
We calculate $\eta^\ast$ by first finding the optimal policy $\pi^\ast$ and then generate Monte-Carlo returns from $\pi^\ast$ for an empirical estimate $\eta^\ast$, followed by a projection onto the $m$ atoms.

We have set $|\mathcal{X}|=5$ throughout but vary $|\mathcal{A}=10|$ for ablations. For each algorithmic variant, we sweep over certain hyper-parameters, such as $\bar{c}\in\{1,2,4\}$ for Retrace and $\lambda\in\{0.1,0.3,0.5,0.7,0.9\}$ for off-policy Q($\lambda$).

\paragraph{The case with $|\mathcal{A}=5|$.} Figure~\ref{fig:qlambda_control_5actions} shows results for when $|\mathcal{A}=5|$ rather than $|\mathcal{A}=20|$ in Figure~\ref{fig:qlambda_control}. We see that since the number of action gets smaller, the problem has become effectively much less off-policy. As a result, Q($\lambda$) becomes stable for all levels of $\lambda$ that we sweep. There is also a general trend of improved contraction and fixed point error as $\lambda$ increases from $0.1$ to $0.9$. Retrace outperforms Q($\lambda$) when $\bar{c}=4$, since the algorithm can effectively make use of the full trace for distributional learning by IS.

\subsection{Deep RL with Atari}

We provide further details on the Atari experiments.

\paragraph{Evaluation.} For the $i$-th of the $57$ Atari games, we obtain the performance of the agent $G_i$ at any given point in training. The normalized performance is computed as 
$Z_i = (G_i - U_i) / (H_i - U_i)$ where $H_i$ is the human performance and $U_i$ is the performance of a random policy. The inter-quartile metric is calculated by dropping out tail samples from across all games and seeds \citep{agarwal2021deep}.

\paragraph{Shared settings for all algorithmic variants.} All algorithmic variants use the same torso architecture as DQN \citep{mnih2015humanlevel} and differ in the head outputs, which we specify below. All agents an Adam optimizer \citep{kingma2014adam} with a fixed learning rate; the optimization is carried out on mini-batches of size $32$ uniformly sampled from the replay buffer. For exploration, the agent acts $\epsilon$-greedy with respect to induced Q-functions, the details of which we specify below. The exploration policy adopts $\epsilon$ that starts with $\epsilon_\text{max}=1$ and linearly decays to $\epsilon_\text{min}=0.01$ over training. At evaluation time, the agent adopts $\epsilon=0.001$; the small exploration probability is to prevent the agent from getting stuck.

For C51, the agent head outputs a matrix of size $|\mathcal{A}| \times m$, which represents the logits to  $\left(p_i(x,a;\theta)\right)_{i=1}^m$. The support $(z_i)_{i=1}^m$ is generated as a uniform array over $[-V_\text{MAX},V_\text{MAX}]$. Though $V_\text{MAX}$ should in theory be determined by $R_\text{MAX}$; in practice, it has been found that setting $V_\text{MAX}=R_\text{MAX}/(1-\gamma)$ leads to highly sub-optimal performance. This is potentially because usually the random returns are far from the extreme values $R_\text{MAX}/(1-\gamma)$, and it is better to set $V_\text{MAX}$ at a smaller value. Here, we set $V_\text{MAX}=10$ and $m=51$.  For details of other hyperparameters, see \citep{bellemare2017distributional}.  The induced Q-function is computed as $Q_\theta(x,a)=\sum_{i=1}^m p_i(x,a;\theta)z_i$. 

\paragraph{Target and behavior policy.} Since the behavior policy $\mu$ is $\epsilon$-greedy, we have access to the probability distributions $\mu(a)$ for each action $a$, which gets stored in the replay buffer along with the transition. At learning time, the algorithms sample from the reply buffer and construct back-up targets. The stored probabilities $\mu(a)$ allow for IS techniques applied in distributional Retrace, when combined with a target policy $\pi$.

For Retrace and one-step, we set $\pi$ as the greedy policy with respect to the Q-function induced by the learner return distribution. For Q($\lambda$), with a fixed $\lambda$, we find that this works sub-optimally. One potential explanation is that the level of off-policyness changes throughout learning, and so a single $\lambda$ might not work optimally across the entire learning process. Instead, we borrow inspirations from the trust region literature, and set the target policy as a mixture of the greedy policy and behavior policy as in Eqn~\eqref{eq:mixing-target}. This allows for better performance for off-policy Q($\lambda$).

\begin{algorithm}[t]
\label{algo:c51-qlambda}
\begin{algorithmic}
\STATE Parameterized categorical distribution $\eta_\theta$ with main network parameter $\theta$ and target network parameter $\theta^-$
\FOR{$k=1,2...$}
\STATE Sample trajectory $(X_t,A_t,R_t)_{t=0}^\infty$ from the replay.
\STATE Compute gradient estimate $\widehat{g}_\theta$ based on Eqn~\eqref{eq:estimate-grad} for the sampled initial state-action pair $(X_0,A_0)$.
\STATE Update parameter $\theta\leftarrow\theta-\kappa \widehat{g}_\theta$.
\STATE Update target parameter $\theta^-\leftarrow(1-\tau)\theta^-+\tau\theta$.
\ENDFOR
\STATE  Output final distribution $\eta_\theta$.
\caption{Q($\lambda$)-C51}
\end{algorithmic}
\end{algorithm}

\end{appendix}

\end{document}